\documentclass{article} 
\usepackage{iclr2026_conference,times}
\usepackage{multirow}

\usepackage{amsmath,amsfonts,bm}

\def\eqref#1{equation~\ref{#1}}

\def\1{\bm{1}}

\DeclareMathAlphabet{\mathsfit}{\encodingdefault}{\sfdefault}{m}{sl}
\SetMathAlphabet{\mathsfit}{bold}{\encodingdefault}{\sfdefault}{bx}{n}

\newcommand{\Var}{\mathrm{Var}}

\usepackage{amsmath,amssymb,amsthm}
\usepackage{mathtools}
\usepackage{tikz}
\usetikzlibrary{arrows.meta, positioning, shapes.geometric, fit, backgrounds}
\usepackage{multicol}
\usepackage{booktabs}
\usepackage{array}
\usepackage{xcolor}
\usepackage{hyperref}
\usepackage{mdframed}
\usepackage{url}
\usepackage{float}
\usepackage{tcolorbox}
\tcbuselibrary{skins, breakable}

\definecolor{theoremblue}{RGB}{220,235,250}
\definecolor{remarkgray}{RGB}{245,245,245}
\definecolor{accentblue}{RGB}{30,80,160}
\definecolor{failred}{RGB}{180,40,40}
\definecolor{successgreen}{RGB}{30,120,60}

\newtheorem{theorem}{Theorem}[section]
\newtheorem{proposition}[theorem]{Proposition}
\newtheorem{lemma}[theorem]{Lemma}

\newtheorem{definition}[theorem]{Definition}
\theoremstyle{remark}
\newtheorem{remark}[theorem]{Remark}
\newtheorem{assumption}{Assumption}

\tcbset{
  mainresultstyle/.style={
    enhanced,
    breakable,
    colback=theoremblue,       
    colframe=accentblue,       
    arc=6pt,                   
    boxrule=1pt,               
    left=10pt, right=10pt,
    top=8pt, bottom=8pt,
    before skip=10pt,
    after skip=10pt,
    fonttitle=\bfseries,
    coltitle=white,
    attach boxed title to top left={
      yshift=-2mm, xshift=8pt
    },
    boxed title style={
      colback=accentblue,
      arc=4pt,
      boxrule=0pt,
    },
  }
}

\newtcolorbox{mainresult}[1][]{
  mainresultstyle,
  title={Main Result},   
  #1
}

\newmdenv[
  backgroundcolor=remarkgray,
  linecolor=gray!50,
  linewidth=0.5pt,
  innerleftmargin=10pt,
  innerrightmargin=10pt,
  innertopmargin=6pt,
  innerbottommargin=6pt,
  skipabove=8pt,
  skipbelow=8pt
]{intuition}

\title{The JEPA Paradox in Language: The Geometry of Linguistic Alternatives}

\author{
Anh Trac Duc Dinh$^{1,2*}$ \& Khang Nhat Hoang Vo$^{3}$
\thanks{Equal contribution. Authors listed alphabetically.}
\thanks{Corresponding author: \texttt{khang.vo@mbzuai.ac.ae}.}
\\
$^{1}$Center for AI Research (CAIR), VinUniversity, Hanoi, Vietnam
\\
$^{2}$Faculty of Computer Science and Engineering\\
Ho Chi Minh City University of Technology (HCMUT), VNU-HCM \\ Ho Chi Minh City, Vietnam \\
\texttt{anh.dinhtracduc@hcmut.edu.vn}
\\
$^{3}$Mohamed bin Zayed University of Artificial Intelligence \\ Abu Dhabi, United Arab Emirates\\
\texttt{khang.vo@mbzuai.ac.ae}
}

\usepackage{lipsum}

\iclrfinalcopy 
\begin{document}

\maketitle

\begin{abstract}
Joint-Embedding Predictive Architectures (JEPAs) are effective for images, video, and audio, yet deterministic JEPA-style latent prediction has not become a standard objective for text encoders. We argue that this gap reflects a mismatch between squared-error latent prediction and the conditional structure of language. The key requirement is conditional concentration: given a context and target location, the target representation should lie near a single meaningful point. Local image prediction often satisfies this through spatial continuity, whereas masked text can admit multiple valid token or span completions whose representations need not share a coherent center. We formalize this mismatch through three conditions---predictability, non-collapse, and low conditional variance---and show how their failure creates centroid degeneracy and collapse pressure in text. Matched I-JEPA and T-JEPA experiments reveal the predicted sequence: mutual-information saturation and elevated target variance precede train--validation instability, effective-rank degeneration, cosine collapse, and poor downstream transfer. The same pattern appears across five independent data seeds, indicating that it is not a sampling artifact. These results do not rule out predictive learning for language; they show that text-compatible JEPA objectives must preserve multiple plausible completions rather than compress them into a single latent point.
\end{abstract}

\section{Introduction}

Self-supervised learning often succeeds by predicting missing information from context. In language, this principle is usually implemented as prediction over discrete symbols or probability distributions, as in masked language modeling (MLM)~\citep{Devlin2019BERT}, causal language modeling (CLM)~\citep{Radford2019GPT2,Brown2020GPT3}, denoising pretraining~\citep{Lewis2020BART,Raffel2020T5}, and replaced-token detection~\citep{Clark2020ELECTRA}. In vision, Joint-Embedding Predictive Architectures instead avoid reconstructing raw inputs: a context encoder observes a masked view, a target encoder embeds the unmasked signal, and a predictor learns to match the target representation in latent space~\citep{LeCun2022APT,Assran2023IJEPA}. This latent-prediction paradigm has achieved strong results on continuous modalities, including images, video, and audio~\citep{Assran2023IJEPA,Bardes2024VJEPA,Tuncay2025AudioJEPA}. Yet deterministic JEPA-style latent prediction has not become a standard recipe for text, where the dominant objectives remain distributional.

\begin{figure}[t]
    \centering
    \includegraphics[width=\linewidth]{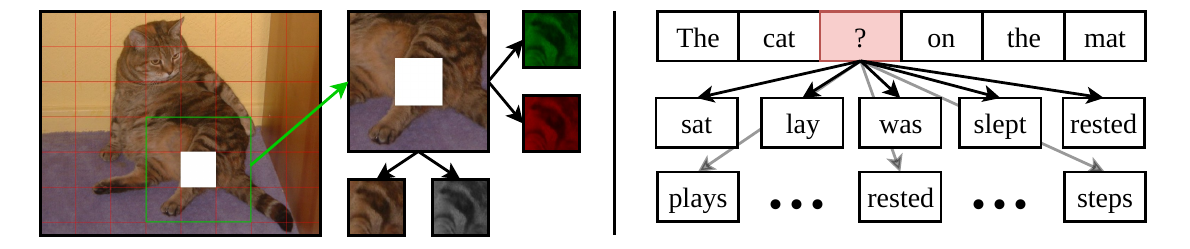}
    \caption{%
   Conditional concentration explains the image-text divide in JEPA.
    In images, spatial continuity constrains masked patches to a low-variance conditional distribution.
    In text, masking leaves many valid continuations 
    whose embeddings occupy distinct directions in representation space, forcing a squared-error
    predictor to regress toward a centroid over heterogeneous continuations rather than a coherent target.}
    \label{fig:conditional-concentration}
\end{figure}

We argue that this gap reflects a statistical mismatch. JEPA is well suited to targets that are \emph{conditionally concentrated}: given a context and a target position, the target representation should lie near a single geometrically meaningful value. Natural images often satisfy this condition locally because nearby pixels and patches are correlated, and masked regions are constrained by surrounding content~\citep{Feige2015Images,He2022MAE,Assran2023IJEPA}. Token-level language does not. The same textual context can admit many valid continuations, reflected in the persistent entropy and perplexity of language even under strong models~\citep{Shannon1951Prediction,Brown2020GPT3}. As illustrated in Figure~\ref{fig:conditional-concentration}, an image patch typically has a low-variance conditional target, whereas a masked token may correspond to multiple valid lexical or semantic continuations whose embeddings occupy different directions. A squared-error JEPA predictor must return a single vector, and therefore regresses toward a centroid over heterogeneous continuations rather than a coherent linguistic target.

We formalize this mismatch through three necessary conditions for useful JEPA learning:
\textbf{predictability}: the target must have low conditional entropy given the context
(spatial smoothness supports this in images; lexical ambiguity imposes an irreducible
floor in text); \textbf{non-collapse}: the objective must force distinct inputs to occupy
distinct regions of representation space (visual diversity creates such pressure, whereas
text permits low-rank collapse without geometric penalty); and \textbf{low conditional
variance}: the reducible loss component must dominate the irreducible one (in images,
residual uncertainty is limited to local texture; in text, multiple valid continuations
dominate the loss and drive the predictor toward centroid degeneracy).

We validate these predictions by constructing a language analogue of JEPA, which we denote T-JEPA.\footnote{We use T-JEPA only as a local abbreviation for our text-based JEPA instantiation. It is unrelated to prior methods with the same name for trajectory similarity computation and tabular representation learning~\citep{Li2024TJEPA,Thimonier2025TJEPA}.}, and comparing it with I-JEPA~\citep{Assran2023IJEPA} under matched training protocols, tracking mutual-information proxies, effective rank, pairwise cosine similarity, train-validation loss divergence, irreducible target variance, and downstream transfer. T-JEPA exhibits early information saturation followed by unstable optimization, rank divergence, cosine-similarity collapse, and poor transfer, while I-JEPA maintains increasing information, non-degenerate rank, and stable generalization -- suggesting that T-JEPA fails not from insufficient capacity or tuning, but because deterministic squared-error latent prediction is misaligned with the conditional geometry of token-level language.

Our contributions are threefold. First, we introduce \emph{conditional concentration} as a condition under which deterministic latent prediction yields useful representations. Second, we identify centroid degeneracy as a failure mode of squared-error latent prediction under ambiguous text continuations. Third, we empirically diagnose T-JEPA through information, rank, cosine-similarity, variance, and downstream metrics, showing that information saturation precedes representation collapse and motivating distributional, contrastive, mixture-based, or semantic-level alternatives. Our results do not imply that predictive representation learning is impossible for language. Rather, they identify the missing inductive bias: language prediction must preserve the multimodal structure of valid continuations instead of compressing them into a single latent centroid.
\section{Related Work}

\paragraph{Self-supervised representation learning.}
Self-supervised learning has developed along several complementary paradigms. Contrastive methods such as CPC/InfoNCE~\citep{Oord2018CPC}, SimCLR~\citep{Chen2020SimCLR}, and MoCo~\citep{He2020MoCo} learn by distinguishing compatible views from negative examples. Non-contrastive methods such as BYOL~\citep{Grill2020BYOL}, Barlow Twins~\citep{Zbontar2021BarlowTwins}, and VICReg~\citep{bardes2022vicreg} remove explicit negatives through architectural asymmetry, stop-gradient mechanisms, or variance--covariance regularization. Joint-Embedding Predictive Architectures~\citep{LeCun2022APT} form a third family: instead of reconstructing inputs or contrasting views, they predict target representations in latent space from a masked context.

\paragraph{JEPA-style latent prediction across modalities.}
I-JEPA demonstrated that masked latent prediction can learn strong visual representations by predicting the embeddings of masked image regions from visible context~\citep{Assran2023IJEPA}. Subsequent work has extended JEPA-style objectives to video and world modeling~\citep{Bardes2024VJEPA,Assran2025VJEPA2,Maes2026LeWorldModel}, audio and speech~\citep{Fei2023AJEPA,Tuncay2025AudioJEPA}, and multimodal text-image settings~\citep{Vo2025TIJEPA,Chen2025VLJEPA}. Recent variants also study stability, regularization, and collapse prevention in JEPA-like objectives~\citep{Balestriero2025LeJEPA,Mo2024CJEPA}. This growing family shows that latent prediction is not limited to images. However, most successful JEPA-style applications involve continuous, spatial, temporal, or cross-modal structure, where masked targets are often constrained by nearby context. Our work asks why the same deterministic latent-prediction recipe has not become a standard general-purpose objective for training text encoders.

\paragraph{Representation collapse and variance preservation.}
Collapse is a central failure mode in self-supervised learning, where representations become constant or occupy a low-dimensional subspace~\citep{Jing2022DimensionalCollapse}. Existing methods avoid collapse through negatives, predictor asymmetry, EMA targets, redundancy reduction, or explicit variance constraints~\citep{Chen2020SimCLR,He2020MoCo,Grill2020BYOL,Zbontar2021BarlowTwins,bardes2022vicreg}. Among these, SIGReg~\citep{Balestriero2025LeJEPA} is a particularly principled and effective collapse-prevention mechanism, but because it constrains only the marginal representation distribution rather than any context-conditional quantity, our analysis in App.~\ref{app:sigreg} shows it is not, by itself, guaranteed to yield predictable or non-degenerate representations when applied to text. JEPA-style methods inherit stabilization from stop-gradient and EMA target encoders, but this stabilization is most effective when the prediction task supplies a structured target signal. We identify a complementary route to collapse: when the conditional target distribution is multimodal, squared-error latent prediction can reduce loss by moving toward a conditional mean or by allowing the encoder to merge distinctions among plausible targets. Thus, our analysis connects collapse to the geometry of the target distribution, not only to optimization dynamics or architectural symmetry.

\paragraph{Language pretraining preserves uncertainty.}
Language pretraining has instead converged on objectives that preserve uncertainty over discrete symbols. Masked language modeling~\citep{Devlin2019BERT}, causal language modeling~\citep{Radford2019GPT2,Brown2020GPT3}, denoising sequence-to-sequence pretraining~\citep{Lewis2020BART,Raffel2020T5}, and replaced-token detection~\citep{Clark2020ELECTRA} predict token distributions, corrupted-token labels, or reconstructed sequences rather than a single continuous target. data2vec also predicts contextualized latent teacher representations from masked inputs and was evaluated across speech, vision, and language~\citep{Baevski2022Data2Vec}, making it an important precursor to modality-general latent prediction. More recent text-involving JEPA variants, such as task-specific Text-JEPA for NL-to-FOL conversion and multimodal TI-JEPA/VL-JEPA models~\citep{Le2026TextJEPA,Vo2025TIJEPA,Chen2025VLJEPA}, show that JEPA-style objectives can incorporate language in specific settings. Two very recent efforts, LLM-JEPA~\citep{huang2026llmjepa} and DLLM-JEPA~\citep{nam2026dllmjepa}, apply a JEPA-style loss directly to text, but in both cases the JEPA term is trained only as an auxiliary addition on top of a distributional generative anchor rather than observed in isolation, so neither resolves why a purely latent-prediction objective transfers from continuous modalities to text at all; we return to a detailed comparison with both systems in App.~\ref{app:discussion-llm-jepa}. They do not remove the central question studied here: when is a masked textual target itself well represented as a single deterministic latent point?

\section{JEPA as Deterministic Latent Prediction}
\label{sec:setup}

We study JEPAs as deterministic latent predictors. A JEPA observes a masked context, encodes the unmasked target with a separate target encoder, and trains a predictor to match the target representation in latent space. This captures the standard image JEPA setting~\citep{Assran2023IJEPA} and defines the text analogue analyzed in this work.

Let $x=(x^{(1)},\ldots,x^{(N)})$ be an input decomposed into $N$ units. For images, units are non-overlapping patches on a two-dimensional grid; for text, units are discrete tokens in a sequence. A masking procedure samples disjoint context and target sets $C,T\subseteq[N]$. The context view $x_C$ is passed through a context encoder $f_\theta$, producing
\begin{equation}
    z_C=f_\theta(x_C).
    \label{eq:context-representation}
\end{equation}
The full unmasked input $x$ is passed through an EMA target encoder $f_{\bar\theta}$, and the target representation at position $j\in T$ is
\begin{equation}
    z_T^{(j)}=f_{\bar\theta}(x)^{(j)}\in\mathbb{R}^d .
    \label{eq:target-representation}
\end{equation}
The target encoder is updated by exponential moving average,
\begin{equation}
    \bar\theta \leftarrow \tau\bar\theta+(1-\tau)\theta, \ \text{with} \ \tau\in(0,1),
    \label{eq:ema-update}
\end{equation}
and is not differentiated through. A predictor $g_\phi$ receives the context representation and a positional query $p_j$, then predicts
\begin{equation}
    \hat z_T^{(j)}=g_\phi(z_C,p_j).
    \label{eq:predicted-target}
\end{equation}

The training objective is the squared-error latent prediction loss
\begin{equation}
  \mathcal{L}_{\mathrm{JEPA}}(\theta,\phi)
  =
  \mathbb{E}_{x,C,T}
  \left[
    \frac{1}{|T|}
    \sum_{j\in T}
    \left\|
      g_\phi(f_\theta(x_C),p_j)
      -
      \operatorname{sg}\!\left(f_{\bar\theta}(x)^{(j)}\right)
    \right\|_2^2
  \right],
  \label{eq:jepa}
\end{equation}
where $\operatorname{sg}(\cdot)$ denotes stop-gradient. Stop-gradient and EMA stabilize training, but the objective does not explicitly require the representation to remain informative, high-rank, or semantically separated.

The key property of Eq.~\ref{eq:jepa} is that it is a point-prediction objective. For fixed $(z_C,p_j)$, the squared-error optimal predictor is the conditional mean
\begin{equation}
    g_\phi^\star(z_C,p_j)
    =
    \mathbb{E}\!\left[z_T^{(j)}\mid z_C,p_j\right].
    \label{eq:conditional-mean-predictor}
\end{equation}
Thus, deterministic JEPA is most suitable when the conditional mean is a representative target. This requires \emph{conditional concentration}: given the context and target location, valid target representations should cluster around a single meaningful point. When this holds, squared-error prediction provides a stable semantic learning signal. Images often satisfy this condition because neighboring patches, object boundaries, textures, and spatial continuity strongly constrain masked regions.

Masked text is less likely to be conditionally concentrated. The same context may admit many valid lexical or semantic completions, while the positional query specifies only where the missing content belongs. The resulting target distribution can therefore be broad or multimodal, making its conditional mean a centroid over incompatible alternatives rather than a coherent completion. Preserving these distinctions weakens point prediction; averaging them away encourages indistinct representations and, ultimately, collapse.

\section{Theoretical Analysis: When Does JEPA Work?}

\label{sec:theory}

We analyze deterministic JEPA through three necessary conditions: predictability (P), non-collapse (NC), and low conditional variance (LV). These conditions are naturally supported in local image prediction, where spatial smoothness and content-dependent visual structure constrain masked targets. In masked text, however, the same context may admit multiple valid token or span completions. If these completions are represented as distinct latent targets, the conditional distribution is not concentrated around a single point; if they are mapped close together, the representation risks collapse. Formal statements and proofs appear in App.~\ref{app:info}- \ref{app:biasvar}; metric definitions appear in App.~\ref{app:metrics}.

The three conditions are not independent failure modes. They describe a single mechanism. If the context does not concentrate the target, the squared-error predictor learns a conditional mean over alternatives. If the representation preserves those alternatives, the loss contains irreducible conditional variance; if the representation removes them, the model reduces loss by merging distinctions. Thus, ambiguity creates pressure toward either noisy point prediction or representational collapse.

\subsection{Argument I: Predictability Requires Conditional Concentration}

\label{sec:arg1}

JEPA can reduce its squared-error loss only when the target representation is predictable from the context. By the Data Processing Inequality, latent representations cannot contain more predictive information than the raw inputs from which they are computed. Thus, the usefulness of deterministic latent prediction is limited by the conditional structure of the data.

For images, local smoothness makes nearby patches informative about one another. Under a Lipschitz target encoder, nearby image patches have nearby target representations:
\begin{equation}
  \mathbb{E}\!\left[
    \left\|
      f_{\bar\theta}(x)^{(j)}
      -
      f_{\bar\theta}(x)^{(i)}
    \right\|_2^2
  \right]
  \leq
  L^2 d(i,j)^2 .
  \label{eq:spatial-concentration-main}
\end{equation}
When the context contains patches near $j$, the masked target is therefore conditionally constrained by visible content. For text, a masked token or span can have several valid completions under the same context. This ambiguity does not automatically imply large latent variance, since an encoder could collapse distinct completions. The key trade-off is therefore: preserving textual distinctions makes the target less predictable as a single point, while removing them reduces uncertainty by discarding information.

\paragraph{Validation.}
We track the InfoNCE mutual-information proxy $\hat I(z_C;z_T)\geq \log N-\mathcal{L}_{\mathrm{InfoNCE}}$~\citep{Poole2019VariationalMI,Oord2018CPC} together with effective rank. A saturated MI proxy with non-trivial rank suggests limited predictive information; simultaneous rank collapse suggests that the model is reducing uncertainty by merging distinctions.

\subsection{Argument II: Non-Collapse Requires a Variance-Preserving Signal}
\label{sec:arg2}

Stop-gradient and EMA target encoders stabilize JEPA training, but the squared-error objective does not explicitly enforce non-degenerate covariance. Whether collapse is discouraged depends on whether the prediction task itself penalizes collapsed representations. In images, a collapsed context representation cannot adapt to content-dependent target variation. If visual targets retain nonzero variation beyond position, then a constant context representation incurs prediction error bounded below by the target variance. This gives image JEPA a natural anti-collapse signal: different visual contexts are needed to predict different masked regions. Text admits a different route. A text encoder can reduce the conditional variance of ambiguous completions by mapping distinct but plausible targets into a lower-dimensional representation. This can preserve low JEPA loss while weakening the distinctions the representation should encode. Thus, collapse in text can arise not only from optimization dynamics, but from the objective's incentive to make multimodal targets easier to predict as a single point.

\paragraph{Validation.}
We track effective rank,
\begin{equation}
    \operatorname{erank}(\Sigma_z)
    =
    \exp\!\left(-\sum_i p_i\log p_i\right)
    \label{eq:erank-main}
\end{equation}
where
\begin{equation}
    p_i=\lambda_i/\sum_j\lambda_j 
\end{equation}
along with mean pairwise cosine similarity, its 95th percentile, and its standard deviation over held-out representations (see  App.~\ref{app:pairwise}). Collapse is indicated by decreasing effective rank, increasing cosine similarity, and shrinking cosine dispersion.

\subsection{Argument III: Squared-Error Prediction Learns a Conditional Mean}
\label{sec:arg3}

The JEPA loss decomposes into reducible approximation error and irreducible conditional variance:
\begin{equation}
  \mathbb{E}
  \left[
    \left\|
      g_\phi(z_C,p_j)-z_T^{(j)}
    \right\|_2^2
  \right]
  =
  \mathbb{E}
  \left[
    \left\|
      g_\phi(z_C,p_j)
      -
      \mathbb{E}[z_T^{(j)}\mid z_C,p_j]
    \right\|_2^2
  \right]
  +
  \mathbb{E}
  \left[
    \operatorname{Var}(z_T^{(j)}\mid z_C,p_j)
  \right].
  \label{eq:bias-variance-main}
\end{equation}
The first term can be reduced by improving the predictor; the second is irreducible for fixed representations. Deterministic JEPA is therefore well matched to conditionally concentrated targets.

For images, the residual uncertainty of a masked patch is often limited to local texture and fine detail, so the conditional variance remains small. For text, let $S$ be a masked span and let $s\sim p(s\mid x_C)$ be a candidate completion. If $h_S(s;x_C)$ denotes the target representation induced by completing the span with $s$, then the squared-error optimal predictor is
\begin{equation}
    g_\phi^\star(z_C,S)
    =
    \sum_s p(s\mid x_C)h_S(s;x_C).
    \label{eq:centroid-main}
\end{equation}
Thus, the predictor returns a centroid over plausible span completions. If multiple completions have non-negligible probability and separated target representations, this centroid need not correspond to any coherent completion. The issue is not autoregression: even when all tokens in a span are predicted in parallel, deterministic squared-error prediction compresses a multimodal conditional distribution into a single Euclidean point.

\paragraph{Validation.}
We estimate $\widehat{\operatorname{Var}}(z^*\mid z_C,S)$ using $K=16$ candidate completions per context. For text, completions are sampled from a frozen masked-language-model oracle and passed through the target encoder. For images, target variability is estimated from augmented variants of the masked region. Estimates are averaged over held-out contexts on both training and validation splits.

\section{Empirical Evidence}
\label{sec:emp_evidence}
\subsection{Experimental Setup}

\label{sec:arg_setup}

We compare image and text instantiations of the same deterministic latent-prediction framework. The goal is not to optimize either modality independently, but to evaluate whether the diagnostics predicted by our theory appear under matched JEPA-style training.

\paragraph{I-JEPA.}
For the image setting, we train an I-JEPA model~\citep{Assran2023IJEPA} on 100K ImageNet-1K images~\citep{Deng2009ImageNet}, using a 90/10 train-validation split and seed 42 . Images are resized to a short side of 256 and randomly cropped to $224\times224$. Both context and target encoders use a ViT-H/16 backbone~\citep{Dosovitskiy2021ViT} with hidden dimension $d=1{,}280$, 32 layers, and 16 attention heads. The target encoder is updated by EMA ($\tau=0.996$). We follow block-wise spatial masking, using encoder scale 0.85-1.0, predictor scale 0.15-0.2, and 4 target blocks per image.

\paragraph{T-JEPA.}
For the text setting, we construct a direct text analogue, T-JEPA, trained on 100K English C4 sentences~\citep{Raffel2020T5}, using a 90/10 train--validation split and maximum sequence length 256. Main-text diagnostics use seed 42. To test robustness, we repeat T-JEPA training with five independent seeds, each resampling its own 100K-sentence subset; the resulting MI, variance, rank, loss, and cosine trajectories are reported in App.~\ref{app:multiseed}. The context and target encoders use a BERT-Large backbone~\citep{Devlin2019BERT} with hidden dimension $d=1{,}024$ and 24 layers. The target encoder is updated by EMA with $\tau=0.996$. The predictor is a 6-layer Transformer with hidden dimension 384 and 8 attention heads. We apply random non-overlapping span masking with 1--5 spans per sentence and span lengths of 1--5 tokens.


\paragraph{Training and diagnostics.} Both models are trained for 15 epochs with AdamW, using a
10-epoch linear warmup from $2\times10^{-4}$ to $10^{-3}$ followed by cosine decay to $10^{-6}$,
gradient clipping at 0.3, and weight decay increasing from 0.04 to 0.4. This schedule was chosen
to make collapse dynamics observable: fixed learning rates or shorter warmups caused T-JEPA to
collapse before the diagnostics could resolve the transition. We monitor five quantities predicted
by Sec.~\ref{sec:theory} -- the InfoNCE mutual-information proxy $\widehat{I}(z_C;z_T)$ (MoCo queue
of 2{,}048, temperature 0.1), effective rank, pairwise cosine similarity, train--validation JEPA
loss, and conditional target variance -- which serve as complementary lenses on a single collapse
dynamic, so their joint convergence on the same transition corroborates one event rather than five
coincidences. Table~\ref{tab:failure-chain} summarizes the predicted temporal failure signature:
T-JEPA does not collapse first and then lose predictability; rather, limited predictability and
elevated conditional variance appear before the representation degenerates.

\begin{table}[t]
\centering
\caption{
Temporal failure signature predicted for T-JEPA. The ordering links the theory to the empirical diagnostics: limited predictability appears before collapse, suggesting that collapse is a response to ambiguous targets rather than the original cause.
}
\label{tab:failure-chain}
\setlength{\tabcolsep}{10pt}
\begin{tabular}{lll}
\toprule
Stage & Diagnostic & Interpretation \\
\midrule
1 & MI proxy saturates & Context does not concentrate the target \\
2 & Conditional variance remains high & Plausible completions remain separated \\
3 & Train-validation loss diverges & Optimization fits unstable signal \\
4 & Effective rank degenerates & Encoder merges distinctions among targets \\
5 & Downstream transfer fails & Representation loses useful text structure \\
\bottomrule
\end{tabular}
\end{table}

\subsection{Argument I Validation}
\label{sec:arg-I_val}

Figure~\ref{fig:mi_rank} compares the evolution of context--target predictability and representation rank for I-JEPA and T-JEPA. Faint curves denote training metrics and bold curves denote validation metrics; our interpretation focuses on the validation split, which is disjoint from training.

\begin{figure}[t]
    \centering
    \includegraphics[width=\linewidth]{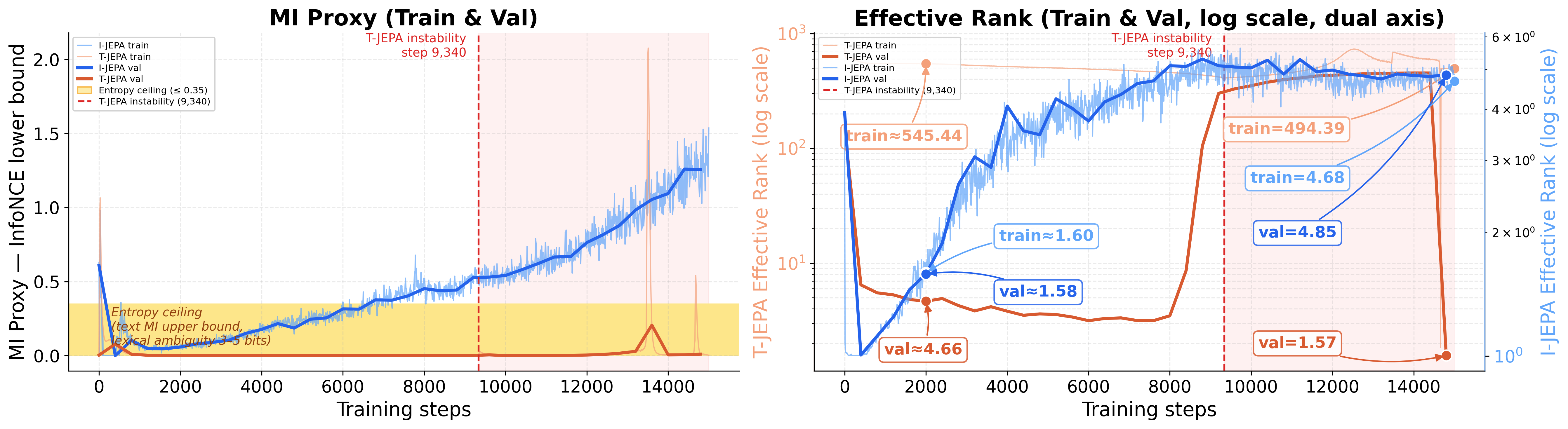}
    \caption{%
        Context--target predictability and representation rank for I-JEPA and T-JEPA.
        Left: InfoNCE mutual-information proxy $\widehat{I}(z_C;\,z_T)$.
        Right: effective rank of the representation covariance $\Sigma_z$.
        Faint curves show training metrics; bold curves show validation metrics.
        The red dashed line marks T-JEPA's instability point at step 9{,}340.%
    }
    \label{fig:mi_rank}
\end{figure}

\paragraph{Predictability saturates before collapse.}
T-JEPA's validation MI proxy remains below $0.35$ nats throughout training, indicating that the masked target representation is only weakly predictable from the visible context. This behavior matches the mechanism in Consequence~I (App. ~\ref{app:info}): masked text can leave several plausible completions, so the target representation need not concentrate around a single latent point. I-JEPA shows the opposite pattern. Its validation MI increases steadily, consistent with local spatial structure making masked image targets more predictable from nearby context.

The ordering of the two diagnostics is critical. T-JEPA's MI proxy saturates early, while its effective rank is still non-trivial (${\approx}4.66$ at step 2{,}000); collapse occurs only later, near the instability point at step 9{,}340. Thus, low predictability is not an artifact of an already-collapsed representation. Instead, predictability fails first, and representation degeneration follows. This supports the central causal chain of our analysis: lack of conditional concentration weakens the latent prediction signal, creating pressure toward later collapse.

\subsection{Argument II Validation}
\label{sec:arg-II_val}

\paragraph{MSE Loss.}
Figure~\ref{fig:t_jepa_instability} (left) shows a clear divergence between the two models.
I-JEPA's losses decrease smoothly in tandem, reflecting the geometric anti-collapse tension of
Prop.~\ref{prop:visual-collapse-cost}.
T-JEPA's losses co-decrease initially, then diverge catastrophically at step 9{,}340:
validation loss spikes while training loss becomes erratic; a pattern of sustained instability
rather than clean descent, consistent with the degenerate solution of
Prop.~\ref{prop:textcollapse}: starved of a meaningful gradient signal by the
irreducible entropy floor, the optimiser memorises batch-specific noise while
generalisation collapses entirely.

\paragraph{Effective Rank.}
Figure~\ref{fig:t_jepa_instability} (right) reveals the representational
mechanism underlying the loss divergence.
I-JEPA's training and validation ranks grow jointly and converge near
$4.7$-$4.85$, confirming that the encoder learns to spread representations
across multiple directions without collapse.
T-JEPA's trajectory is starkly different: the \emph{validation} rank
stabilises near $4.66$ from step $\sim\!2{,}000$, then undergoes a sharp
upward spike at step $9{,}340$, briefly plateaus, before collapsing
catastrophically to $1.57$, while the \emph{training} rank artificially
inflates to $494.39$-$545.44$.
This train-val divergence directly instantiates
$\lambda_{\min}(\Sigma_z)\!\to\!0$ on the generalisation distribution as
predicted by Prop.~\ref{prop:textcollapse}: the training objective is
simultaneously gamed by overfitting to noise, while the encoder loses all
representational structure on held-out data.
Loss and rank together constitute the collapse signature:
training loss \emph{decreases} precisely \emph{as} validation rank collapses,
confirming the degenerate solution is loss-free.

\begin{figure}[t]
    \centering
    \includegraphics[width=\linewidth]{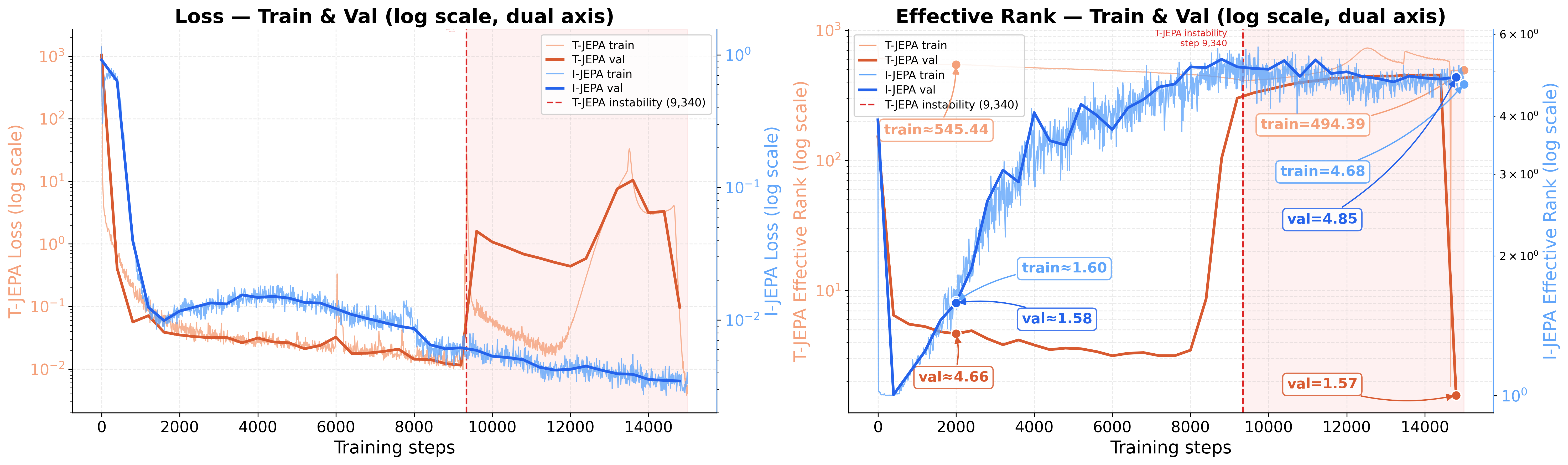}
    \caption{
        Loss and Effective Rank for I-JEPA and T-JEPA
        (training and validation, log scale).
        \emph{Left:} JEPA training and validation loss curves.
        \emph{Right:} Effective rank of $\Sigma_z$.
        Red dashed line marks T-JEPA's instability point (step 9{,}340).%
    }
    \label{fig:t_jepa_instability}
\end{figure}

\subsection{Argument III Validation}
\label{sec:arg-III_val}

I-JEPA's irreducible variance initially falls sharply to 0.0008 (on both train and val splits) near step 1,500 (Figure~\ref{fig:irred_var} ). However, instead of a monotonic decline, it rebounds and plateaus alongside T-JEPA until roughly step 7,000, before decisively dropping again to establish a clear gap. By step 9,340, I-JEPA's variance falls to 0.0062 (train) and 0.0059 (val), eventually stabilising near 0.0017 for both splits - a small residual consistent with $\sigma^2_{\mathrm{texture}}$ (Prop.~\ref{prop:imagevar}). T-JEPA's variance, meanwhile, remains elevated at 0.0149 (train) and 0.0116 (val) at the middle step. This represents roughly double I-JEPA's level at this step, reflecting the persistent irreducible floor imposed by lexical ambiguity (Prop.~\ref{prop:textvar}) that no amount of encoder updates can reduce. After step 9,340, T-JEPA's variance drops sharply, contracting to 0.0010 (train) and 0.0012 (val) between steps 17,500 and 21,000 - not because ambiguity resolves, but because the encoder degenerates and compresses all token embeddings into a narrow region, consistent with the rank collapse reported in Sec.~\ref{sec:arg-II_val}. This clear gap during the stable phase, followed by artificial compression at collapse, provides direct empirical support for Consequence~III (App.~\ref{app:biasvar}). The same qualitative ordering is reproduced across five independent T-JEPA runs: predictability remains limited before effective-rank degeneration, while the late reduction in target variance coincides with directional collapse rather than improved target concentration (App.~\ref{app:multiseed}).

\begin{figure}[t]
    \centering
    \includegraphics[width=\linewidth]{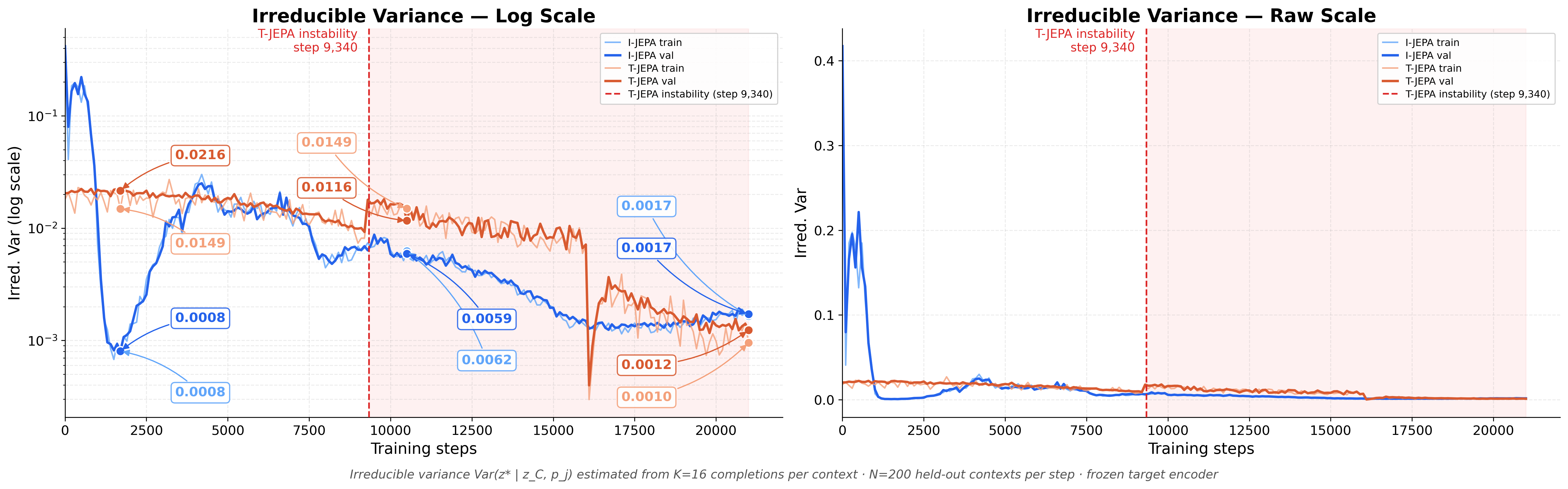}
    \caption{
        Irreducible variance $\widehat{\Var}(z^*\mid z_C, p_j)$ for
        I-JEPA and T-JEPA (log and raw scale).
        T-JEPA's irreducible variance remains persistently higher than
        I-JEPA's throughout training and collapses to near-zero only after
        representational collapse at step ${\sim}15{,}000$.
        The red dashed line marks T-JEPA's instability point (step 9{,}340).%
    }
    \label{fig:irred_var}
\end{figure}

\subsection{Downstream Transfer Across NLP Tasks}
\label{sec:downstream}

\begin{table}[t]
\centering
\caption{
Downstream performance of T-JEPA and self-supervised baselines across classification and retrieval benchmarks.
\textbf{Bold}: best result; \underline{underline}: best among SSL methods.
Retrieval metrics are reported in \%.
$^\text{M}$: mask augmentation; $^\text{R}$: replace augmentation.
Mean $\pm$ std are reported over 5 independent runs.
}
\label{tab:downstream}
\resizebox{\textwidth}{!}{%
\fontsize{8}{9}\selectfont
\setlength{\tabcolsep}{4pt}
\begin{tabular}{lcccccccc}
\toprule
\multirow{2}{*}{\textbf{Model}}
  & \multicolumn{2}{c}{\textbf{IMDB}}
  & \multicolumn{2}{c}{\textbf{SNLI}}
  & \multicolumn{2}{c}{\textbf{MTEB/FEVER}}
  & \multicolumn{2}{c}{\textbf{MTEB/MSMARCO}} \\
\cmidrule(lr){2-3}\cmidrule(lr){4-5}\cmidrule(lr){6-7}\cmidrule(lr){8-9}
  & \textbf{Acc.} & \textbf{F1}
  & \textbf{Acc.} & \textbf{F1}
  & \textbf{nDCG@10} & \textbf{R@100}
  & \textbf{nDCG@10} & \textbf{R@100} \\
\midrule
BERT
& \textbf{81.91} $\pm$ 1.35 & \textbf{81.89} $\pm$ 2.43
& \textbf{59.72} $\pm$ 1.25 & \textbf{59.11} $\pm$ 2.31
& \textbf{14.43} $\pm$ 1.17 & \textbf{38.82} $\pm$ 1.24
& \textbf{9.494} $\pm$ 0.93 & \textbf{33.672} $\pm$ 1.37 \\
\midrule
Barlow Twins$^\text{M}$
& \underline{64.34} $\pm$ 2.26 & 64.99 $\pm$ 1.29
& \underline{45.55} $\pm$ 1.31  & \underline{45.99} $\pm$ 2.59
& 0 $\pm$ 0.00 & 0 $\pm$ 0.00
& \underline{0.1082} $\pm$ 0.009 & \underline{0.419} $\pm$ 0.096 \\
VICReg$^\text{M}$
& 63.12 $\pm$ 2.13 & \underline{68.85} $\pm$ 2.25
& 43.93 $\pm$ 2.27 & 41.76 $\pm$ 1.14
& 0 $\pm$ 0.00 & 0 $\pm$ 0.00
& 0.009 $\pm$ 0.021 & 0.115 $\pm$ 0.035 \\
BYOL$^\text{M}$
& 50.23 $\pm$ 0.72 & 66.63 $\pm$ 1.08
& 34.28 $\pm$ 1.12 & 17.02 $\pm$ 0.99
& 0 $\pm$ 0.00 & 0 $\pm$ 0.00
& 0 $\pm$ 0.00 & 0 $\pm$ 0.00 \\
\midrule
Barlow Twins$^\text{R}$
& 62.02 $\pm$ 1.33 & 63.85 $\pm$ 1.36
& 43.87 $\pm$ 1.18 & 43.10 $\pm$ 2.21
& 0 $\pm$ 0.00 & 0 $\pm$ 0.00
& 0.009 $\pm$ 0.032 & 0.201 $\pm$ 0.042 \\
VICReg$^\text{R}$
& 59.10 $\pm$ 1.22 & 66.56 $\pm$ 2.12
& 41.59 $\pm$ 1.41 & 41.14 $\pm$ 1.31
& 0 $\pm$ 0.00 & 0 $\pm$ 0.00
& 0.004 $\pm$ 0.001 & 0.006 $\pm$ 0.002 \\
BYOL$^\text{R}$
& 61.77 $\pm$ 1.17 & 67.39 $\pm$ 1.24
& 43.99 $\pm$ 2.28 & 44.04 $\pm$ 1.23
& 0 $\pm$ 0.00 & 0 $\pm$ 0.00
& 0.021 $\pm$ 0.004 & 0.358 $\pm$ 0.012 \\
\midrule
T-JEPA
& 50.02 $\pm$ 0.13 & 66.68 $\pm$ 0.41
& 33.34 $\pm$ 0.12 & 16.69 $\pm$ 0.62
& 0 $\pm$ 0.00 & 0 $\pm$ 0.00
& 0 $\pm$ 0.00 & 0 $\pm$ 0.00 \\
\bottomrule
\end{tabular}
}
\end{table}

We next test whether the intrinsic failures observed above translate into poor downstream representations. We compare T-JEPA with BERT and three non-contrastive self-supervised baselines: Barlow Twins, VICReg, and BYOL. All models are pretrained on 3 million English C4 sentences using the \texttt{bert-base-uncased} tokenizer with maximum sequence length 256, we also use this model as backbone for all baselines. After pretraining, each encoder is frozen and used as a fixed feature extractor. We evaluate on two classification benchmarks, IMDB and SNLI~\citep{maas-EtAl:2011:ACL-HLT2011,bowman-etal-2015-large}, and two retrieval benchmarks from MTEB, FEVER and MSMARCO~\citep{muennighoff-etal-2023-mteb}. Full hyperparameters and protocol details are reported in App.~\ref{app:setup}.

Because text lacks the augmentation diversity of images, we evaluate the SSL baselines under two corruption schemes: \emph{Mask}, where both views are independently span-masked at 15--20\% of tokens, and \emph{Replace}, where one view is the original sentence and the other has 15--20\% of tokens replaced via a frozen pretrained BERT proposal distribution (constrained to differ from the originals). T-JEPA uses the same 15--20\% span-masking budget.

Table~\ref{tab:downstream} shows that BERT, trained with MLM, substantially outperforms all self-supervised baselines across classification and retrieval. This is consistent with our hypothesis: distributional token prediction preserves uncertainty over plausible completions, whereas deterministic latent regression compresses that uncertainty into a single target. Among the SSL baselines, Barlow Twins under masking performs best on most classification metrics, suggesting that explicit redundancy reduction is comparatively robust to span corruption. In contrast, T-JEPA reaches chance-level classification and zero retrieval performance, matching the intrinsic collapse diagnostics in Sections~\ref{sec:arg-I_val}--\ref{sec:arg3}.

The contrast between BYOL$^\text{M}$ and BYOL$^\text{R}$ further supports the role of conditional ambiguity. BYOL with independent masking collapses similarly to T-JEPA, whereas BYOL with replacement corruption remains competitive with VICReg and Barlow Twins. Replacement preserves most of the sentence and keeps the two views globally aligned, reducing the need to resolve multiple missing lexical completions. Masking, by contrast, exposes the predictor to the same multi-hypothesis ambiguity that drives centroid degeneracy in T-JEPA. The activation-landscape visualizations in App.~\ref{app:rep_act} provide a qualitative counterpart: BERT retains rich inter-token and intra-vector variation; VICReg, Barlow Twins, and BYOL$^\text{R}$ retain moderate structure; BYOL$^\text{M}$ is flatter; and T-JEPA is nearly uniform and low-amplitude.

\section{Conclusion}
\label{sec:conclusion}

Deterministic JEPA-style latent prediction works when masked targets have a meaningful conditional center. Images often satisfy this through spatial smoothness, whereas masked text admits multiple valid completions, forcing T-JEPA toward either centroid prediction or representational collapse. Our experiments show that limited predictability and elevated variance precede rank degeneration, cosine collapse, and poor transfer. Predictive learning for language must therefore preserve multiple plausible completions; related LLM-JEPA, DLLM-JEPA variants avoid this issue by retaining a generative objective alongside JEPA (App.~\ref{app:discussion-llm-jepa}).

\bibliography{iclr2026_conference}
\bibliographystyle{iclr2026_conference}

\appendix
\clearpage
\section{Appendix Overview}
\label{app:overview}
\setlength{\tabcolsep}{0.5pt}
\begin{table}[H]
\centering
\begin{tabular}{llr}
\toprule
\textbf{Appendix} & \textbf{Contents} & \textbf{Page} \\
\midrule
A & Appendix Overview (this section) & \pageref{app:overview} \\
B & Formal Proofs & \pageref{app:info} \\
\quad B.1 & \quad Proofs of Argument I: The Predictability Bound & \pageref{app:info} \\
\quad B.2 & \quad Proofs of Argument II: Representation Collapse & \pageref{app:collapse} \\
\quad B.3 & \quad Proofs of Argument III: Bias-Variance Decomposition & \pageref{app:biasvar} \\

C & Empirical Metrics: Definitions and Rationale & \pageref{app:metrics} \\

D & Multi-Seed Robustness Of The Collapse Diagnostics & \pageref{app:multiseed}\\

E & Argument II's Supplementary Metric: Pairwise Cosine Similarity & \pageref{app:pairwise} \\

F & Experimental Details for Section~\ref{sec:downstream} & \pageref{app:setup}
\\

G & Visualization: The Phenomenon of Centroid Degeneration & \pageref{app:centroid}\\ 

H & Visualization:  Representation Activation Heatmaps Across Self-Supervised Models & \pageref{app:rep_act} \\
I & Visualization: Patterns of T-JEPA's Representation Collapse &\pageref{app:vis_collapse} \\
J & Discussion: JEPA on Top of LLMs versus JEPA Directly on Text & \pageref{app:discussion-llm-jepa} \\
K & Discussion: SIGReg Prevents Collapse But Does Not Resolve Predictability or  &  \\ 
 & Centroid Degeneracy & \pageref{app:sigreg} 
\\
L & Mathematical Notations used throughout the paper and appendices & \pageref{app:notation} \\
\bottomrule
\end{tabular}
\caption{Structure of the appendix.}
\label{tab:appendix_overview}
\end{table}

\section{Formal Proofs}

\subsection{Proofs of Argument I: The Predictability Bound}
\label{app:info}

We first formalize the sense in which JEPA requires the masked target representation to be predictable from the visible context. Let
\begin{equation}
    z_C = f_\theta(x_C)
\end{equation}
and
\begin{equation}
    z_T^{(j)} = f_{\bar\theta}(x)^{(j)}
\end{equation}
denote the context representation and the target representation at position $j$. Since the predictor receives only $z_C$ and a positional query $p_j$, the JEPA loss can be small only when $z_T^{(j)}$ is nearly determined by $(z_C,p_j)$. Equivalently, the conditional uncertainty
\begin{equation}
    H\!\left(z_T^{(j)} \mid z_C,p_j\right)
\end{equation}
must be small relative to the target uncertainty at that position. By the chain rule of mutual information,
\begin{equation}
  I\!\left(z_C;\, z_T^{(j)} \mid p_j\right)
  =
  H\!\left(z_T^{(j)} \mid p_j\right)
  -
  H\!\left(z_T^{(j)} \mid z_C,p_j\right).
  \label{eq:mi}
\end{equation}
Thus, useful deterministic latent prediction requires the context representation to remove a substantial fraction of the target uncertainty.

\begin{lemma}[Data-processing ceiling]
\label{lem:dpi}
For deterministic encoders $f_\theta$ and $f_{\bar\theta}$,
\begin{equation}
    I\!\left(z_C;\, z_T^{(j)}\right)
    \leq
    I\!\left(x_C;\, x\right).
    \label{eq:dpi-global}
\end{equation}
If the target representation at position $j$ is local, in the sense that $z_T^{(j)}$ is a deterministic function of $x^{(j)}$, then
\begin{equation}
    I\!\left(z_C;\, z_T^{(j)}\right)
    \leq
    I\!\left(x_C;\, x^{(j)}\right).
    \label{eq:dpi-local}
\end{equation}
\end{lemma}

\begin{proof}
Because $z_C=f_\theta(x_C)$ is a deterministic function of $x_C$ and
$z_T^{(j)}=f_{\bar\theta}(x)^{(j)}$ is a deterministic function of the full input $x$, the pair
$(z_C,z_T^{(j)})$ is obtained from $(x_C,x)$ by deterministic post-processing. The Data Processing Inequality therefore gives Eq.~\ref{eq:dpi-global}: deterministic encoders cannot increase the mutual information available between the raw context and the raw input.

Under the additional locality assumption, there exists a deterministic map $h_{\bar\theta,j}$ such that
\begin{equation}
    z_T^{(j)} = h_{\bar\theta,j}\!\left(x^{(j)}\right).
\end{equation}
Hence $z_C$ is a deterministic function of $x_C$, while $z_T^{(j)}$ is a deterministic function of $x^{(j)}$. Applying the Data Processing Inequality to these two post-processings gives
\begin{equation}
    I\!\left(f_\theta(x_C);\, h_{\bar\theta,j}(x^{(j)})\right)
    \leq
    I\!\left(x_C;\, x^{(j)}\right),
\end{equation}
which is Eq.~\ref{eq:dpi-local}.
\end{proof}

\begin{remark}
Lemma~\ref{lem:dpi} gives an information ceiling, not a collapse result. It says that the latent representations cannot contain more predictive information than is already present in the raw context-target relationship. The image-text difference therefore cannot be resolved by encoder capacity alone: it depends on how strongly the visible context constrains the missing target.
\end{remark}

\subsubsection{Images: local geometry narrows the target}

Natural images exhibit local spatial regularity: nearby pixels and patches tend to be correlated, and neighboring patches often belong to the same object, boundary, or texture region~\citep{Feige2015Images}. We use this regularity to formalize the sense in which visible image context can geometrically restrict a masked target.

\begin{proposition}[Spatial concentration for images]
\label{prop:lip}
Let $x$ be a natural image divided into patches on a two-dimensional grid, and let $d(i,j)$ denote the distance between patch centers. Assume that the image distribution has bounded local variation and that the target representation is locally Lipschitz with respect to patch-level perturbations on the data manifold. Then there exists a constant $L>0$ such that, for nearby patches $i$ and $j$,
\begin{equation}
    \mathbb{E}\!\left[
    \left\|
        f_{\bar\theta}(x)^{(j)}
        -
        f_{\bar\theta}(x)^{(i)}
    \right\|_2^2
    \right]
    \leq
    L^2 d(i,j)^2 .
    \label{eq:lip}
\end{equation}
\end{proposition}

\begin{proof}
Bounded local variation means that nearby image patches have bounded expected discrepancy:
\begin{equation}
    \mathbb{E}\!\left[
    \left\|x^{(j)}-x^{(i)}\right\|_2^2
    \right]
    \leq
    G^2 d(i,j)^2
\end{equation}
for some constant $G>0$. By local Lipschitzness of the target representation on the data manifold, there exists $L_f>0$ such that
\begin{equation}
    \left\|
        f_{\bar\theta}(x)^{(j)}
        -
        f_{\bar\theta}(x)^{(i)}
    \right\|_2
    \leq
    L_f
    \left\|
        x^{(j)}-x^{(i)}
    \right\|_2 .
\end{equation}
Squaring both sides, taking expectation, and setting $L=L_fG$ gives Eq.~\ref{eq:lip}.
\end{proof}

\begin{remark}
Proposition~\ref{prop:lip} does not claim that images are deterministic or that every masked patch has a unique completion. It states that local visual context restricts the range of plausible target representations. When visible patches lie near the masked region, the conditional distribution of $z_T^{(j)}$ given $(z_C,p_j)$ is expected to be concentrated up to residual texture-level uncertainty. This is the regime in which squared-error latent prediction is well aligned with the data.
\end{remark}

\subsubsection{Text: lexical alternatives disperse the target}

For text, the target unit $x^{(j)}$ is a discrete token from a vocabulary $\mathcal{V}$. Even when the surrounding context is informative, several continuations may remain valid. This creates a prediction geometry different from images: uncertainty is not merely local noise around one value, but a set of plausible lexical alternatives.

\paragraph{Empirical lexical uncertainty.}
For a masked token $x^{(j)}$ and context $x_C$, the conditional uncertainty is
\begin{equation}
    H\!\left(x^{(j)} \mid x_C\right)
    =
    -\sum_{v\in\mathcal{V}}
    p(v\mid x_C)\log_2 p(v\mid x_C).
    \label{eq:textent}
\end{equation}
A conditional perplexity of $8$-$32$ corresponds to $3$-$5$ bits of uncertainty per token, and classical estimates of English entropy support the broader claim that natural language retains uncertainty even under strong contextual constraints~\citep{Shannon1951Prediction}. The precise value depends on tokenization, domain, and context length; the important point is that the conditional distribution is often not concentrated on a single continuation.

To connect lexical uncertainty to latent-space uncertainty, we require the encoder to preserve distinctions among plausible completions. Without such a condition, an encoder could trivially make the prediction problem easy by mapping distinct tokens to the same vector, which is precisely the collapse mechanism analyzed below.

\begin{assumption}[Target separation]
\label{assump:target-separation}
For a context $x_C$, define the set of plausible continuations
\begin{equation}
    \mathcal{V}_\epsilon(x_C)
    =
    \{v \in \mathcal{V}: p(v\mid x_C)>\epsilon\}.
\end{equation}
The target encoder preserves distinctions among plausible continuations if there exists $\Delta>0$ such that, for all distinct $u,v\in\mathcal{V}_\epsilon(x_C)$,
\begin{equation}
    \left\|
        f_{\bar\theta}(x_{j\leftarrow u})^{(j)}
        -
        f_{\bar\theta}(x_{j\leftarrow v})^{(j)}
    \right\|_2
    \geq
    \Delta,
    \label{eq:target-separation}
\end{equation}
where $x_{j\leftarrow v}$ denotes the sequence obtained by placing token $v$ at position $j$ while keeping the context fixed.
\end{assumption}

\begin{proposition}[Non-concentration of separated text targets]
\label{prop:text-nonconcentration}
Fix a context $x_C$ and suppose there exist at least two plausible continuations $u,v\in\mathcal{V}_\epsilon(x_C)$ whose target representations satisfy Assumption~\ref{assump:target-separation}. Then the conditional latent target distribution cannot have zero variance around a single point. In particular, any deterministic squared-error predictor must either incur nonzero conditional variance or rely on an encoder that reduces the separation between plausible continuations.
\end{proposition}

\begin{proof}
For a fixed context, the latent target is the random variable induced by
\begin{equation}
    v \sim p(\cdot\mid x_C),
    \qquad
    z_T^{(j)} = f_{\bar\theta}(x_{j\leftarrow v})^{(j)} .
\end{equation}
If two plausible continuations have non-negligible probability and their target representations are separated by at least $\Delta$, then $z_T^{(j)}$ places mass on at least two separated points. Therefore it cannot be concentrated at a single vector. Its expected squared error around any deterministic prediction is minimized by the conditional mean,
\begin{equation}
    \mathbb{E}\!\left[z_T^{(j)}\mid z_C,p_j\right],
\end{equation}
but the residual conditional variance is nonzero whenever separated alternatives retain nonzero mass. The only way for this variance to vanish is for the encoder to map the plausible continuations to the same, or nearly the same, latent vector, which violates the separation condition and removes the distinctions the representation was meant to preserve.
\end{proof}

\begin{mainresult}[title=Consequence I: Predictability Requires Concentration]
For images, local spatial structure restricts the masked target to a low-variance set of plausible representations. For token-level text, the same context can leave several valid lexical alternatives. If the encoder preserves these alternatives, the latent target remains non-concentrated; if it removes them, the representation loses distinctions among valid completions. Deterministic JEPA prediction on text therefore faces a trade-off between irreducible prediction variance and loss-reducing collapse.
\end{mainresult}

\subsection{Detailed Proofs for Argument II: Representation Collapse}
\label{app:collapse}

\subsubsection{The collapse problem}

The JEPA objective in Eq.~\ref{eq:jepa} contains no explicit constraint requiring the representation distribution to remain full-rank. Stop-gradient and EMA stabilize the target branch, but they do not by themselves rule out degenerate representations. The most extreme solution is a constant representation where
\begin{equation}
  f_\theta(x_C)=\mathbf{c},
  \qquad
  f_{\bar\theta}(x)^{(j)}=\mathbf{c},
  \qquad
  g_\phi(\mathbf{c},p_j)=\mathbf{c},
\end{equation}
for all contexts $x_C$, inputs $x$, and target positions $j$, where $\mathbf{c}\in\mathbb{R}^d$ is fixed. This solution achieves zero JEPA loss while encoding no information about the input. More generally, partial collapse occurs when the encoder maps inputs to a low-dimensional subset of the representation space.

\begin{definition}[Representation covariance and collapse]
\label{def:cov}
Let $z=f_\theta(x_C)$, where $x_C$ is drawn from the data distribution. The representation covariance is
\begin{equation}
    \Sigma_z
    =
    \mathbb{E}\!\left[
        (z-\mathbb{E}[z])(z-\mathbb{E}[z])^\top
    \right]
    \in \mathbb{R}^{d\times d}.
\end{equation}
We say that the representation collapses along a direction $u\in\mathbb{R}^d$ if
\begin{equation}
    \operatorname{Var}(u^\top z)
    =
    u^\top\Sigma_z u
    \to 0 .
\end{equation}
Full collapse corresponds to $\Sigma_z=0$. Dimensional collapse corresponds to variance vanishing along one or more directions, equivalently to a loss of effective rank in $\Sigma_z$.
\end{definition}

\subsubsection{A sufficient condition for non-collapse}

\begin{lemma}[Non-collapse via directional spread]
\label{lem:noncollapse}
If there exists $\gamma>0$ such that for every unit vector $u\in\mathbb{R}^d$,
\begin{equation}
    \operatorname{Var}(u^\top z)
    =
    u^\top\Sigma_z u
    \geq
    \gamma,
    \label{eq:spread}
\end{equation}
then $\lambda_{\min}(\Sigma_z)\geq\gamma$, and the representation is non-collapsed.
\end{lemma}

\begin{proof}
By the variational characterization of eigenvalues,
\[
    \lambda_{\min}(\Sigma_z)
    =
    \min_{\|u\|_2=1} u^\top\Sigma_z u .
\]
If Eq.~\ref{eq:spread} holds for every unit vector $u$, then this minimum is at least $\gamma$. Conversely, if there exists a unit vector $u$ such that $\operatorname{Var}(u^\top z)=0$, then $u^\top z$ is constant almost surely, so the representation carries no variation along that direction.
\end{proof}

The question is therefore whether the prediction objective supplies pressure for directional spread. We argue that local image prediction provides such pressure through content-dependent targets, whereas token-level text admits a collapse route because many distinct continuations can be made equivalent under the latent loss.

\subsubsection{Images: content-dependent prediction discourages collapse}

Natural images contain spatially local structure, but they are not determined by position alone. A patch at a fixed location can contain different objects, textures, colors, and boundaries across images. If the context encoder collapses, the predictor receives no image-dependent information from $z_C$ and is left only with the positional query $p_j$. Such a predictor cannot explain content-dependent variation in the target representation.

\begin{assumption}[Non-degenerate visual targets]
\label{assump:visual-target-variance}
For each target position $j$, the target representation has nonzero variation beyond position:
\begin{equation}
    \mathbb{E}
    \left[
    \left\|
        z_T^{(j)}
        -
        \mathbb{E}\!\left[z_T^{(j)}\mid p_j\right]
    \right\|_2^2
    \right]
    \geq
    \sigma_{\mathrm{vis}}^2
\end{equation}
for some $\sigma_{\mathrm{vis}}^2>0$.
\end{assumption}

\begin{proposition}[Collapsed image context incurs prediction error]
\label{prop:visual-collapse-cost}
Suppose Assumption~\ref{assump:visual-target-variance} holds. If the context encoder collapses to a constant representation $z_C=\mathbf{c}$, then any JEPA predictor incurs loss at least $\sigma_{\mathrm{vis}}^2$:
\begin{equation}
    \inf_{g_\phi}
    \mathbb{E}
    \left[
    \left\|
        g_\phi(\mathbf{c},p_j)
        -
        z_T^{(j)}
    \right\|_2^2
    \right]
    \geq
    \sigma_{\mathrm{vis}}^2 .
\end{equation}
\end{proposition}

\begin{proof}
If $z_C=\mathbf{c}$ for all inputs, then the predictor can depend only on the positional query $p_j$. Under squared-error loss, the optimal predictor is the conditional mean
\begin{equation}
    g_\phi^\star(\mathbf{c},p_j)
    =
    \mathbb{E}\!\left[z_T^{(j)}\mid p_j\right].
\end{equation}
The minimum achievable loss is therefore
\begin{equation}
    \mathbb{E}
    \left[
    \left\|
        z_T^{(j)}
        -
        \mathbb{E}\!\left[z_T^{(j)}\mid p_j\right]
    \right\|_2^2
    \right],
\end{equation}
which is at least $\sigma_{\mathrm{vis}}^2$ by Assumption~\ref{assump:visual-target-variance}.
\end{proof}

\begin{remark}
Proposition~\ref{prop:visual-collapse-cost} does not claim that image JEPA is immune to collapse in all architectures or optimization regimes. It states that, when target features vary with image content, a collapsed context representation cannot match the prediction loss achievable by a content-aware representation. Spatial locality strengthens this anti-collapse pressure: neighboring visible patches reduce uncertainty about the target, so encoding visual content in $z_C$ can lower the loss beyond what position alone permits.
\end{remark}

\subsubsection{Text: collapse can reduce ambiguity}

For token-level text, the failure mode is different. A masked context often admits several valid continuations. If the target encoder preserves these distinctions, the predictor faces a non-concentrated target distribution. If the encoder removes them, the target distribution becomes easier to predict, but the representation loses information. The JEPA loss itself does not prevent this second route.

Consider the sentences
\begin{center}
    ``The cat sat on the mat.''
\end{center}
and
\begin{center}
    ``The dog lay on the rug.''
\end{center}
They differ lexically, but share a coarse semantic pattern. A representation that preserves lexical identity must separate them; a representation that keeps only coarse semantics may map them nearby or even identically. The problem is not that such invariance is always bad. The problem is that a deterministic squared-error objective has no intrinsic mechanism for deciding which distinctions should be preserved.

\begin{proposition}[Low-rank text encodings can minimize JEPA loss]
\label{prop:textcollapse}
Let $\pi:\mathcal{X}\to\{1,\ldots,K\}$ be a coarse partition of text inputs into equivalence classes, with $K<d$. Suppose the context encoder factors through this partition,
\begin{equation}
    f_\theta(x_C)=\mathbf{e}_{\pi(x_C)},
\end{equation}
where $\{\mathbf{e}_1,\ldots,\mathbf{e}_K\}\subset\mathbb{R}^d$. Then the representation covariance of $z_C=f_\theta(x_C)$ has rank at most $K-1$, and is therefore rank-deficient whenever $K<d$. Moreover, if the target representation also depends only on the same class and target position,
\begin{equation}
    f_{\bar\theta}(x)^{(j)}=\mathbf{t}_{\pi(x_C),j},
\end{equation}
then there exists a predictor $g_\phi$ that attains zero JEPA loss on this representation.
\end{proposition}

\begin{proof}
Since $f_\theta(x_C)$ takes values only in the finite set $\{\mathbf{e}_1,\ldots,\mathbf{e}_K\}$, the support of $z_C$ lies in the affine span of at most $K$ points. Therefore,
\begin{equation}
    \operatorname{rank}(\Sigma_z)\leq K-1.
\end{equation}
If $K<d$, the covariance is rank-deficient.

Now assume the target representation is determined by the same equivalence class and target position. Define the predictor on each class by
\begin{equation}
    g_\phi(\mathbf{e}_k,p_j)=\mathbf{t}_{k,j}.
\end{equation}
Then, for every example,
\begin{equation}
    g_\phi(f_\theta(x_C),p_j)
    =
    f_{\bar\theta}(x)^{(j)}.
\end{equation}
Hence the JEPA loss is zero, despite the context representation being supported on at most $K$ points. If the factorization holds only approximately, the same construction yields low loss while the representation remains low-rank.
\end{proof}

\begin{remark}
Proposition~\ref{prop:textcollapse} does not claim that semantic compression is always undesirable. Coarse invariance can be useful for many tasks. The issue is that token-level language contains lexical, syntactic, and semantic distinctions that may all matter, while the JEPA loss only rewards predictability of the target vector. When merging alternatives reduces conditional variance, squared-error latent prediction can favor a lower-rank representation even if it discards useful linguistic structure.
\end{remark}

\begin{mainresult}[title=Consequence II: Collapse as a Low-Variance Escape Route]
For images, a collapsed context representation leaves the predictor with position alone and incurs content-prediction error whenever target features vary with visual content. For token-level text, the objective can reduce ambiguity by mapping several valid continuations or contexts into a lower-dimensional equivalence class. Thus, collapse is not merely an optimization accident: it is a low-variance escape route made available by deterministic latent prediction.
\end{mainresult}

\subsection{Detailed Proofs for Argument III: Bias--Variance Decomposition}
\label{app:biasvar}

\subsubsection{Decomposition of the JEPA loss}

For a fixed context representation and target position, write
\begin{equation}
    z_C = f_\theta(x_C),
    \qquad
    z^* = z_T^{(j)} = f_{\bar\theta}(x)^{(j)} .
\end{equation}
Even after conditioning on $(z_C,p_j)$, the target $z^*$ may remain random because the visible context need not uniquely determine the masked content. Define the conditional mean target
\begin{equation}
    \bar z^{(j)}(z_C,p_j)
    =
    \mathbb{E}\!\left[z^* \mid z_C,p_j\right].
\end{equation}

\begin{proposition}[Bias--variance decomposition of JEPA loss]
\label{prop:bv}
For any deterministic predictor $g_\phi$, the per-target JEPA loss decomposes as
\begin{align}
  \mathbb{E}\!\left[
    \left\|g_\phi(z_C,p_j)-z^*\right\|_2^2
  \right]
  &=
  \mathbb{E}\!\left[
    \left\|g_\phi(z_C,p_j)-\bar z^{(j)}(z_C,p_j)\right\|_2^2
  \right]
  \nonumber \\
  &\quad+
  \mathbb{E}\!\left[
    \left\|z^*-\bar z^{(j)}(z_C,p_j)\right\|_2^2
  \right].
  \label{eq:bv}
\end{align}
The first term is reducible approximation error. The second term is the irreducible conditional variance of the target representation.
\end{proposition}

\begin{proof}
Add and subtract the conditional mean:
\begin{equation}
  g_\phi(z_C,p_j)-z^*
  =
  g_\phi(z_C,p_j)-\bar z^{(j)}(z_C,p_j)
  +
  \bar z^{(j)}(z_C,p_j)-z^* .
\end{equation}
Expanding the squared norm gives two squared-norm terms and a cross-term. Conditioning on $(z_C,p_j)$, the cross-term vanishes because $g_\phi(z_C,p_j)$ and $\bar z^{(j)}(z_C,p_j)$ are deterministic functions of $(z_C,p_j)$, while
\begin{equation}
    \mathbb{E}\!\left[
        \bar z^{(j)}(z_C,p_j)-z^*
        \mid z_C,p_j
    \right]
    =
    0 .
\end{equation}
Taking expectation over $(z_C,p_j)$ gives Eq.~\ref{eq:bv}.
\end{proof}

\begin{remark}
Eq.~\ref{eq:bv} is the central reason squared-error latent prediction requires conditional concentration. The predictor can reduce the first term by approximating the conditional mean, but the second term remains whenever the target distribution is not concentrated around that mean. Thus, a deterministic JEPA objective is well aligned only when the conditional mean is representative of the target distribution.
\end{remark}

\subsubsection{Images: local context yields concentrated targets}

In local image prediction, visible neighboring patches strongly constrain the missing target. The remaining uncertainty is not zero, but it is often confined to visually compatible texture, color, or boundary details. Thus, the conditional mean remains a meaningful target rather than an average over unrelated alternatives.

\begin{assumption}[Image target concentration]
\label{assump:image-concentration}
For natural images under the masking scheme considered, the target representation satisfies
\begin{equation}
    \mathbb{E}
    \left[
      \operatorname{Var}
      \left(
        z_T^{(j)} \mid z_C,p_j
      \right)
    \right]
    \leq
    \sigma_{\mathrm{img}}^2 ,
    \label{eq:imagevar}
\end{equation}
where $\sigma_{\mathrm{img}}^2$ is small relative to the reducible prediction error during the informative phase of training.
\end{assumption}

\begin{proposition}[Image JEPA has a useful reducible signal]
\label{prop:imagevar}
Under Assumption~\ref{assump:image-concentration}, the irreducible component of the image JEPA loss is bounded by $\sigma_{\mathrm{img}}^2$. Consequently, whenever
\begin{equation}
    \mathbb{E}
    \left[
      \left\|
        g_\phi(z_C,p_j)-\bar z^{(j)}(z_C,p_j)
      \right\|_2^2
    \right]
    \gg
    \sigma_{\mathrm{img}}^2 ,
\end{equation}
loss reduction is dominated by the reducible approximation term, so improving the predictor and context representation yields a meaningful training signal.
\end{proposition}

\begin{proof}
By Proposition~\ref{prop:bv}, the JEPA loss decomposes into a reducible approximation term and an irreducible conditional-variance term. Assumption~\ref{assump:image-concentration} bounds the irreducible term by $\sigma_{\mathrm{img}}^2$. When the approximation term is much larger than this bound, changes that improve $g_\phi(z_C,p_j)$ as an estimate of $\bar z^{(j)}(z_C,p_j)$ dominate changes in the total loss.
\end{proof}

\begin{remark}
Assumption~\ref{assump:image-concentration} does not require deterministic reconstruction of the masked patch. It only requires that visible context restrict the target representation enough that the conditional mean is representative. Spatial smoothness and local continuity make this plausible for images: nearby visible patches narrow the set of compatible target representations, leaving residual uncertainty mostly in fine texture or detail.
\end{remark}

\subsubsection{Text: multi-hypothesis conditional targets}

For token-level text, a fixed context can admit several valid continuations. This is not a problem for distributional language modeling: masked or causal language models predict a distribution over the vocabulary, so multiple valid tokens can all receive probability mass. The issue arises for deterministic latent prediction, where the model must output a single vector under squared-error loss.

Let $x_{j\leftarrow v}$ denote the sequence obtained by inserting token $v$ at position $j$ while keeping the observed context fixed. Define the target representation induced by continuation $v$ as
\begin{equation}
    h_j(v;x_C)
    =
    f_{\bar\theta}(x_{j\leftarrow v})^{(j)} .
\end{equation}
Under the conditional token distribution $v\sim p(\cdot\mid x_C)$, the latent target is the random variable
\begin{equation}
    z^* = h_j(v;x_C).
\end{equation}

\begin{proposition}[Conditional variance under multiple valid continuations]
\label{prop:textvar}
For a fixed context $x_C$, the irreducible variance of the text target is
\begin{equation}
    \operatorname{Var}(z^*\mid z_C,p_j)
    =
    \sum_{v\in\mathcal{V}}
    p(v\mid x_C)
    \left\|
      h_j(v;x_C)
      -
      \bar z^{(j)}(z_C,p_j)
    \right\|_2^2 ,
    \label{eq:textvar}
\end{equation}
where
\begin{equation}
    \bar z^{(j)}(z_C,p_j)
    =
    \sum_{v\in\mathcal{V}}
    p(v\mid x_C) h_j(v;x_C)
\end{equation}
is the conditional mean. If the distribution assigns nonzero mass to multiple separated continuations, then this variance is nonzero.
\end{proposition}

\begin{proof}
The expression follows from the definition of conditional variance for the random variable $z^*=h_j(v;x_C)$ with $v\sim p(\cdot\mid x_C)$. If there exist two continuations $u$ and $v$ with $p(u\mid x_C)>0$, $p(v\mid x_C)>0$, and
\begin{equation}
    \left\|
        h_j(u;x_C)-h_j(v;x_C)
    \right\|_2
    >
    0,
\end{equation}
then the conditional distribution of $z^*$ is not a point mass. Therefore its variance around the conditional mean is strictly positive.
\end{proof}

A quantitative lower bound follows under a separation condition.

\begin{assumption}[Separated plausible continuations]
\label{assump:separated-continuations}
For a context $x_C$, suppose there exist $m\geq2$ plausible continuations $v_1,\ldots,v_m$ such that
\begin{equation}
    p(v_i\mid x_C)\geq \alpha
    \qquad
    \text{for all } i,
\end{equation}
and their target representations are pairwise separated:
\begin{equation}
    \left\|
        h_j(v_i;x_C)-h_j(v_\ell;x_C)
    \right\|_2
    \geq
    \Delta
    \qquad
    \text{for } i\neq \ell .
\end{equation}
\end{assumption}

\begin{proposition}[Lower bound from separated alternatives]
\label{prop:textvar-lower}
Under Assumption~\ref{assump:separated-continuations},
\begin{equation}
    \operatorname{Var}(z^*\mid z_C,p_j)
    \geq
    \frac{\alpha^2 m(m-1)}{2}\Delta^2 .
\end{equation}
\end{proposition}

\begin{proof}
For any random vector $Z$ with an independent copy $Z'$,
\begin{equation}
    \operatorname{Var}(Z)
    =
    \frac{1}{2}
    \mathbb{E}\!\left[\|Z-Z'\|_2^2\right],
\end{equation}
where $\operatorname{Var}(Z)=\mathbb{E}\|Z-\mathbb{E}Z\|_2^2$. Applying this identity to $Z=z^*$ and restricting the expectation to the ordered pairs $(v_i,v_\ell)$ with $i\neq \ell$ gives
\begin{equation}
    \operatorname{Var}(z^*\mid z_C,p_j)
    \geq
    \frac{1}{2}
    \sum_{i\neq \ell}
    p(v_i\mid x_C)p(v_\ell\mid x_C)
    \Delta^2 .
\end{equation}
Since each $p(v_i\mid x_C)\geq\alpha$, the bound follows.
\end{proof}

The squared-error optimal predictor is therefore
\begin{equation}
    g_\phi^\star(z_C,p_j)
    =
    \bar z^{(j)}(z_C,p_j)
    =
    \sum_{v\in\mathcal{V}}
    p(v\mid x_C)h_j(v;x_C).
    \label{eq:centroid}
\end{equation}
For text, this point is a centroid over plausible continuations. Unless the continuations are already close in representation space, the centroid need not correspond to any actual continuation. This clarifies why language modeling objectives avoid the same failure mode: they preserve uncertainty by predicting a distribution over tokens, whereas deterministic latent MSE prediction compresses that distribution into one conditional mean. Reducing the variance requires either a distributional or mixture-valued predictor, or an encoder that brings distinct continuations closer together. The latter route is precisely the collapse mechanism analyzed in Sec.~\ref{app:collapse}.

\begin{mainresult}[title=Consequence III: Centroid Degeneracy]
Squared-error latent prediction learns a conditional mean. This is appropriate when the target distribution is conditionally concentrated, as in local image prediction. For token-level text, plausible continuations can be separated in representation space. The conditional mean then becomes a centroid over alternatives rather than a representation of a single coherent target. Thus, deterministic latent prediction faces a bias--variance trade-off: preserve distinctions and incur irreducible variance, or reduce variance by collapsing the distinctions the representation should encode.
\end{mainresult}

\section{Empirical Metrics: Definitions and Rationale}
\label{app:metrics}

Each of the three arguments in Sec.~\ref{sec:theory} is validated by one or more
diagnostic metrics tracked throughout training on held-out validation batches.
We define each metric precisely below and explain why it constitutes evidence for
the corresponding argument.

\subsection{Metrics for Argument I: Mutual Information Proxy}
\label{app:mi_proxy}

\paragraph{InfoNCE MI proxy.}
Direct estimation of $I(z_C;z_T)$ is intractable in high-dimensional representation spaces. We therefore use the InfoNCE objective as a relative proxy for context--target predictability~\citep{Poole2019VariationalMI,Oord2018CPC}. For a batch or queue of $N$ candidate targets, define
\begin{equation}
  \widehat{I}_{\mathrm{NCE}}(z_C;z_T)
  =
  \log N
  -
  \mathcal{L}_{\mathrm{InfoNCE}},
  \label{eq:mi_proxy}
\end{equation}
where
\begin{equation}
  \mathcal{L}_{\mathrm{InfoNCE}}
  =
  -\,\mathbb{E}\!\left[
    \log
    \frac{
      \exp\!\left(s(z_C,z_T)/\tau\right)
    }{
      \frac{1}{N}\sum_{k=1}^{N}
      \exp\!\left(s(z_C,z_T^{(k)})/\tau\right)
    }
  \right].
  \label{eq:infonce}
\end{equation}
Here $s(\cdot,\cdot)$ is the similarity function; in our experiments we use the dot product, $s(z_C,z_T)=z_C^\top z_T$. Negatives $z_T^{(k)}$ are drawn from a MoCo-style queue of size $N=2{,}048$, with temperature $\tau=0.1$.

The proxy is high when $z_C$ can reliably identify its paired target $z_T$ among many distractors, and low when the context contains limited information about the target representation. We use it only as a comparative diagnostic across training: a sustained plateau indicates that context--target predictability has saturated. Under our hypothesis, T-JEPA should plateau early because masked text often admits several plausible continuations, whereas I-JEPA should continue improving as spatial context increasingly constrains the masked image target.
\subsection{Metrics for Argument II: Effective Rank and Cosine Similarity}

\paragraph{Effective rank.}
Given the representation covariance
\begin{equation}
    \Sigma_z
    =
    \mathbb{E}\!\left[
        (z-\mathbb{E}[z])(z-\mathbb{E}[z])^\top
    \right],
\end{equation}
let $\lambda_1\geq\lambda_2\geq\cdots\geq\lambda_d\geq0$ be its eigenvalues. We define the effective rank as the exponential of the spectral entropy:
\begin{equation}
  \mathrm{erank}(\Sigma_z)
  =
  \exp\!\left(
    -\sum_{i=1}^{d} p_i \log p_i
  \right),
  \qquad
  p_i = \frac{\lambda_i}{\sum_{j=1}^{d} \lambda_j}.
  \label{eq:erank}
\end{equation}
When the spectrum is spread across many directions, $\mathrm{erank}(\Sigma_z)$ is large; when most variance is concentrated in a few directions, it is small. A drop in effective rank is therefore a spectral signature of dimensional collapse, as predicted by Proposition~\ref{prop:textcollapse}. In the degenerate case where all eigenvalues vanish, the representation is constant; in practice, we compute effective rank with standard numerical stabilization. Effective rank is computed from full-sequence context representations every 10 training steps.

\paragraph{Pairwise cosine similarity.}
As a complementary geometry-level collapse indicator, we track pairwise cosine similarities over a random subsample of $N=2{,}048$ held-out representations:
\begin{equation}
  \bar{s}
  =
  \frac{2}{N(N-1)}
  \sum_{i < j}
  \frac{z_i^\top z_j}{\|z_i\|_2\,\|z_j\|_2}.
  \label{eq:cosine}
\end{equation}
We also report the 95th percentile $s_{95}$ and standard deviation $\sigma_s$ of the pairwise cosine distribution. Non-collapsed representations occupy multiple directions, so $\bar{s}$ remains bounded away from $1$ and the distribution retains nontrivial spread. Under directional collapse, representations align, driving $\bar{s}\to1$, $s_{95}\to1$, and $\sigma_s\to0$. Together, effective rank and cosine statistics distinguish low-dimensional spectral collapse from mere changes in representation scale. Full results are reported in App.~\ref{app:pairwise}.
\subsection{Metrics for Argument III: Irreducible Variance Estimator}

\paragraph{Estimating conditional target variance.}
We estimate the irreducible component in Proposition~\ref{prop:bv} by sampling multiple plausible targets for the same context and measuring their dispersion in target-representation space. For each held-out context and target position, let
\begin{equation}
    z_1^*,\ldots,z_K^*
\end{equation}
denote $K$ target representations obtained from alternative completions or perturbations that are compatible with the same visible context. We estimate
\begin{equation}
    \widehat{\operatorname{Var}}(z^*\mid z_C,p_j)
    =
    \sum_{k=1}^{K}
    w_k
    \left\|
        z_k^*-\bar z^*
    \right\|_2^2,
    \qquad
    \bar z^*
    =
    \sum_{k=1}^{K} w_k z_k^*,
    \label{eq:var-estimator}
\end{equation}
where $w_k=1/K$ for uniformly sampled variants and $\sum_k w_k=1$ for importance-weighted completions. We use $K=16$ variants per context, average over $N=200$ held-out contexts, and log the estimate every 100 training steps on both train and validation splits.

\paragraph{Images.}
For images, alternative targets are generated by applying small pixel-space perturbations to the same image while keeping the mask fixed:
\begin{equation}
  x_k = \operatorname{clip}\!\left(x+\varepsilon_k,\,-1,\,1\right),
  \qquad
  \varepsilon_k \sim \mathcal{N}(0,\sigma_{\mathrm{aug}}^2 I),
  \label{eq:aug}
\end{equation}
with $\sigma_{\mathrm{aug}}=0.2$, together with sparse sign-flip perturbations. Each perturbed image is passed through the EMA target encoder $f_{\bar\theta}$, and the target embeddings at the masked positions are mean-pooled to obtain $z_k^*\in\mathbb{R}^{1280}$. We then apply Eq.~\ref{eq:var-estimator} with uniform weights.

\paragraph{Text.}
For text, alternative targets are sampled from a frozen masked-language-model oracle. Given the visible context $x_C$, the oracle defines a proposal distribution over span completions. We sample $K=16$ completions whose token probabilities exceed $\varepsilon_{\min}=0.001$, fill the masked span to obtain completed sequences $x_{S\leftarrow s_k}$, and pass each completed sequence through the EMA target encoder:
\begin{equation}
    z_k^*
    =
    \frac{1}{|S|}
    \sum_{j\in S}
    f_{\bar\theta}(x_{S\leftarrow s_k})^{(j)} .
\end{equation}
Completions are weighted by their oracle probabilities,
\begin{equation}
    w_k
    =
    \frac{
        q(s_k\mid x_C)
    }{
        \sum_{\ell=1}^{K} q(s_\ell\mid x_C)
    },
    \qquad
    q(s_k\mid x_C)
    =
    \prod_{j\in S}
    p(v_k^{(j)}\mid x_C).
\end{equation}
We then compute Eq.~\ref{eq:var-estimator} using these normalized weights.

A persistently larger text variance during the stable phase indicates that multiple plausible completions remain separated in target-representation space. A later drop in this variance is not by itself evidence that ambiguity has disappeared; when accompanied by rank degeneration and cosine collapse, it indicates that the encoder has reduced variance by compressing the representations themselves.

\section{Multi-Seed Robustness Of The Collapse Diagnostics}
\label{app:multiseed}

To verify that the collapse signature reported for T-JEPA under seed 42 in Sec.~\ref{sec:arg_setup} is not an artifact of a single data sample, we repeat \emph{T-JEPA-only} training with five independent seeds (0--4), each re-sampling its own 100K-sentence subset from C4 (each seed has it's own 100K sample sentences); no I-JEPA comparison is included in this appendix, as the goal here is solely to check within-model stability of the diagnostics across data draws. All five T-JEPA runs reproduce the same qualitative failure chain observed for seed 42: the MI proxy stays near zero, well below the entropy ceiling, throughout training (Figure~\ref{fig:mi_rank_5_seed}, left); effective rank stays near-degenerate for the first five epochs, rises sharply and jointly across seeds between epochs 6--9, then collapses back to near 1 by epoch 12--15 (Figure~\ref{fig:mi_rank_5_seed}, right; Figure~\ref{fig:loss_rank_5_seed}, right); validation loss shows sharp transient spikes in the same epoch 9--11 window before settling to a low, degenerate value once rank collapses (Figure~\ref{fig:loss_rank_5_seed}, left); irreducible variance declines monotonically with a common late-training contraction coinciding with rank collapse rather than resolved ambiguity (Figure~\ref{fig:irr_5_seed}); and pairwise cosine similarity stays pinned near 1.0, with only a brief seed-dependent dip around epochs 9--12 (Figure~\ref{fig:cosine_5_seed}). These results confirm that the predictability-saturation-before-collapse ordering established for T-JEPA in the main text is a property of it's objective itself, not of any particular data draw.

\begin{figure}[H]
    \centering
    \includegraphics[width=\linewidth]{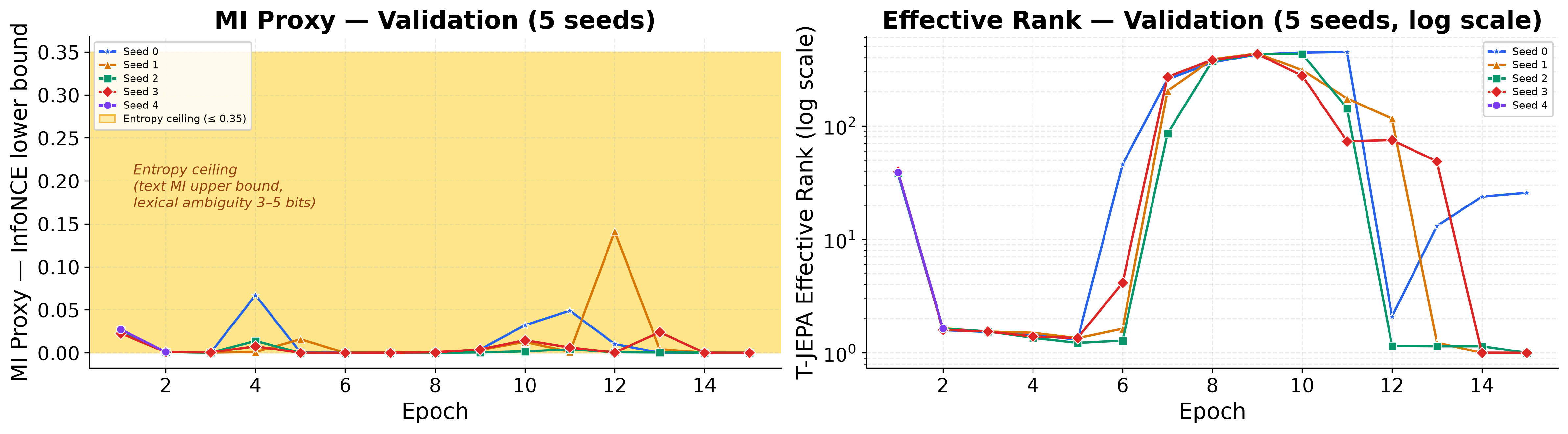}
    \caption{T-JEPA validation MI proxy (left) and effective rank (right, log scale) across 5 seeds ($0$–$4$), each independently re-sampling 100K C4 sentences, confirming early predictability saturation is not a sampling artifact.}
    \label{fig:mi_rank_5_seed}
\end{figure}
\begin{figure}[H]
    \centering
    \includegraphics[width=\linewidth]{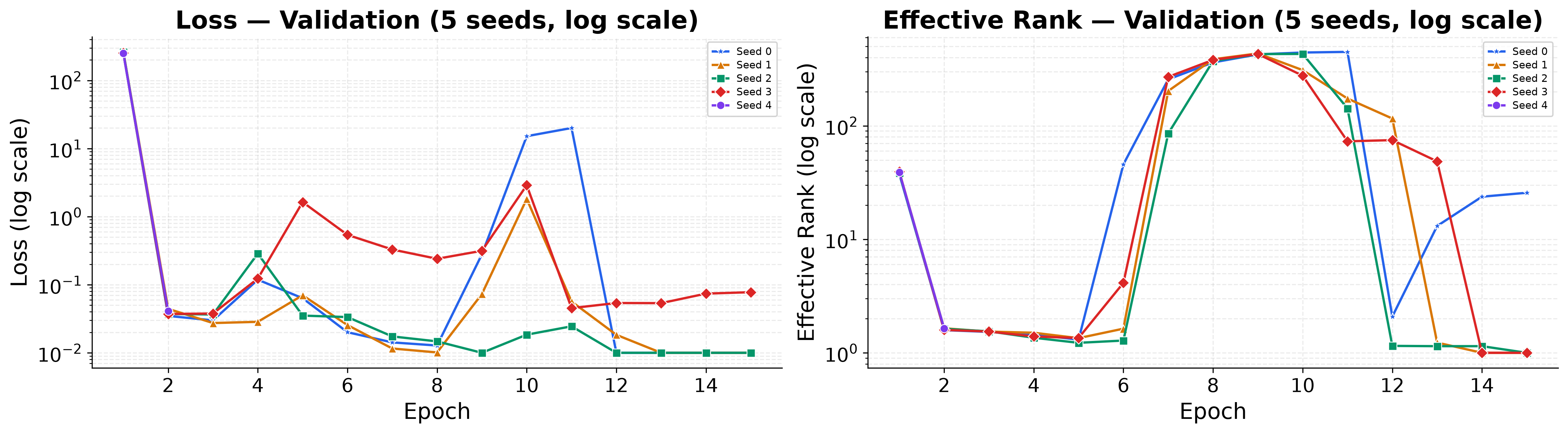}
    \caption{T-JEPA validation loss (left) and effective rank (right, both log scale) across 5 seeds ($0$–$4$), each independently re-sampling 100K C4 sentences, confirming the loss–rank collapse signature is robust to sampling variation.}
    \label{fig:loss_rank_5_seed}
\end{figure}
\begin{figure}[H]
    \centering
    \includegraphics[width=\linewidth]{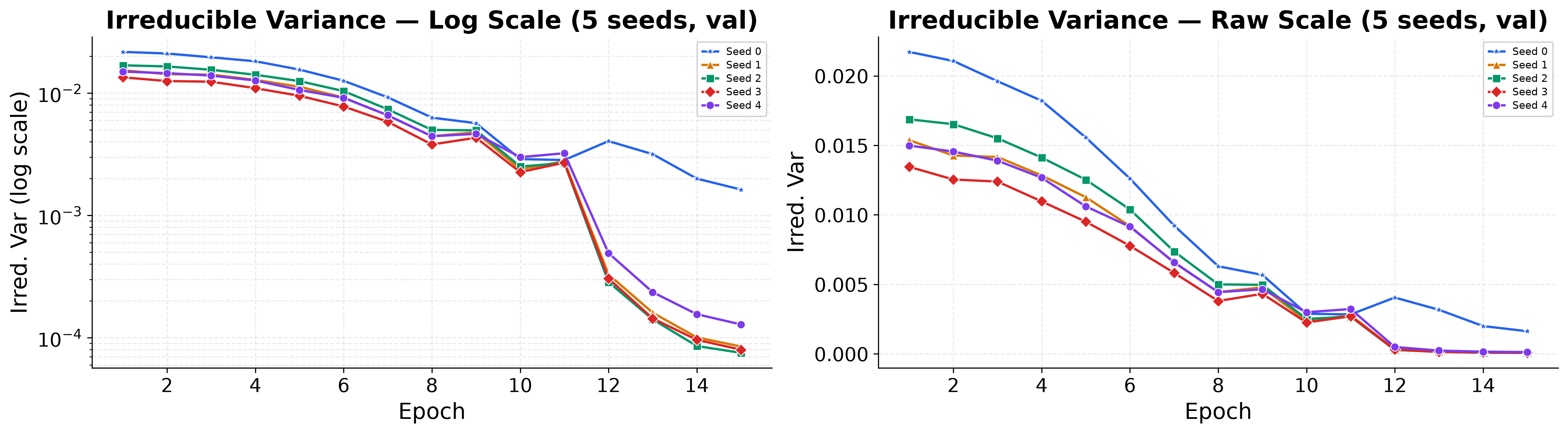}
    \caption{Irreducible variance $\widehat{\mathrm{Var}}(z^* \mid z_C, p_j)$ for T-JEPA (log scale, left; raw scale, right) across 5 seeds ($0$–$4$), each independently re-sampling 100K C4 sentences, showing the elevated conditional-variance floor is not seed-specific.}
    \label{fig:irr_5_seed}
\end{figure}
\begin{figure}[t]
    \centering
    \includegraphics[width=\linewidth]{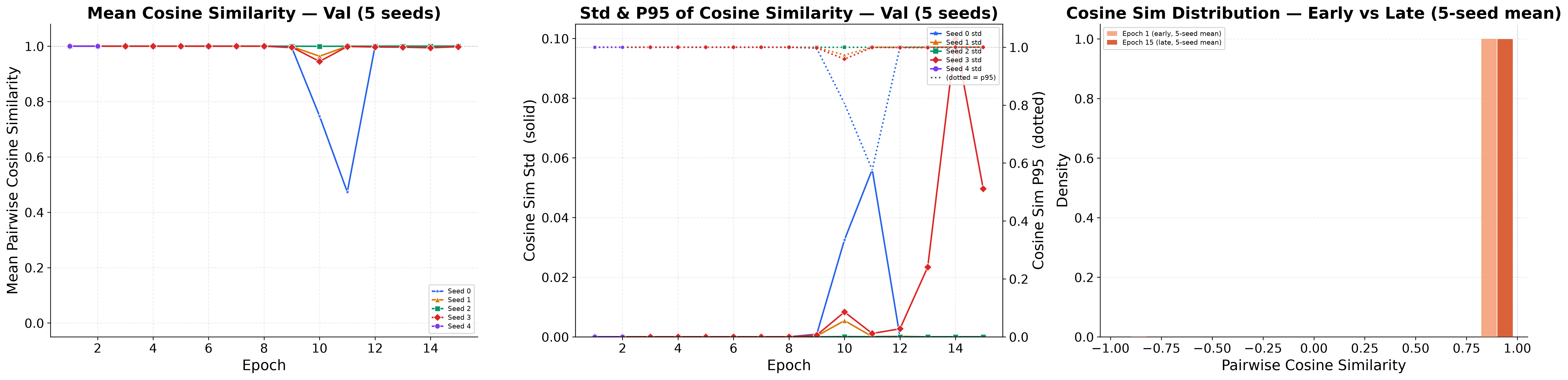}
    \caption{Pairwise cosine-similarity statistics for T-JEPA across 5 seeds ($0$–$4$, each independently re-sampling 100K C4 sentences): mean similarity (left), std/p95 (center), and early-vs-late histogram averaged over seeds (right), confirming directional collapse is robust to sampling variation.}
    \label{fig:cosine_5_seed}
\end{figure}

\section{Supplementary Metric for Argument II (Seed 42)}
\label{app:pairwise}

\begin{figure}[t]
    \centering
    \includegraphics[width=\linewidth]{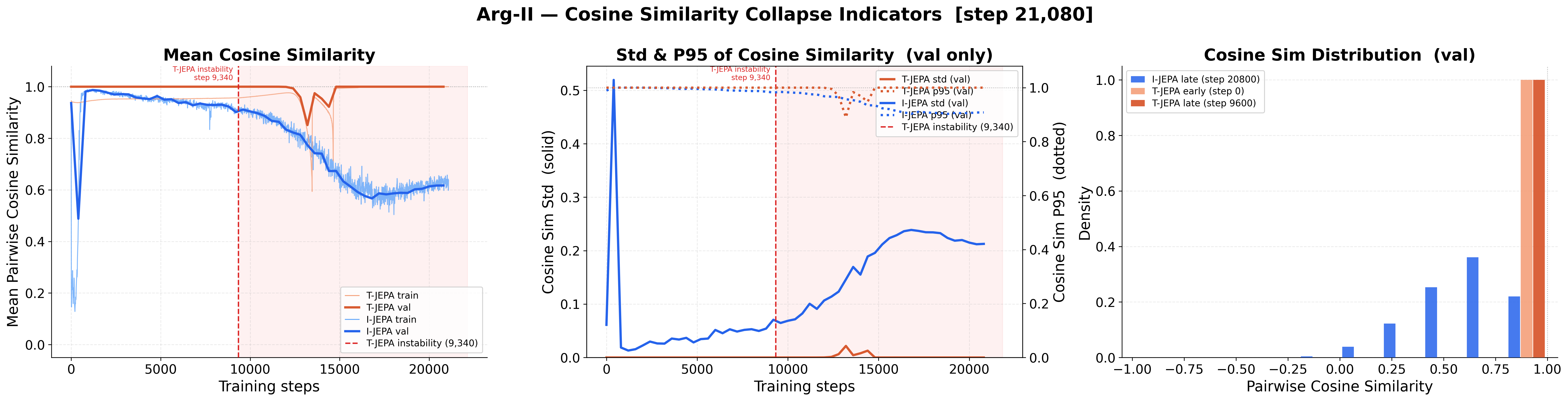}
    \caption{%
        \textbf{Pairwise cosine similarity of I-JEPA and T-JEPA representations.}
        Left: mean pairwise cosine similarity over training.
        T-JEPA rises to ${\approx}1.0$ after the instability point at step 9{,}340, while I-JEPA remains bounded away from $1$.
        Centre: standard deviation and 95th percentile of validation cosine similarities.
        After collapse, T-JEPA has vanishing dispersion and near-unit upper quantiles.
        Right: late-training histogram of pairwise cosine similarities.
        T-JEPA concentrates at $1.0$, indicating directional collapse, while I-JEPA maintains a broad distribution.%
    }
    \label{fig:cosine_collapse}
\end{figure}

\paragraph{Cosine similarity.}
Figure~\ref{fig:cosine_collapse} provides a geometry-level confirmation of the rank-collapse diagnostics. I-JEPA maintains a broad pairwise-cosine distribution, with substantial dispersion throughout training, indicating that different inputs continue to occupy diverse directions in representation space. T-JEPA behaves differently: after the instability point at step 9{,}340, its mean pairwise cosine similarity approaches $1$, its dispersion vanishes, and its late-training histogram concentrates at the right edge. Thus, the drop in effective rank is not merely a change in representation scale; held-out inputs are mapped to nearly the same direction in latent space.

\section{Experimental Details for Section~\ref{sec:downstream}}
\label{app:setup}

\paragraph{Dataset and benchmarks.}
All models are pretrained on 3M English C4 sentences with maximum sequence length 256, using the \texttt{bert-base-uncased} tokenizer. After pretraining, encoders are frozen and evaluated as fixed feature extractors. We report classification performance on IMDB and SNLI using accuracy and F1, and retrieval performance on MTEB/FEVER and MTEB/MSMARCO using nDCG@10 and Recall@100. Mean $\pm$ std are reported over 5 independent runs.

\paragraph{Baselines and corruptions.}
We compare T-JEPA against BERT, VICReg, Barlow Twins, and BYOL. BERT is trained with standard masked language modeling. The two-view SSL baselines are trained under two text corruption schemes. In the Mask setting ($^\text{M}$), both views are independently corrupted by random non-overlapping span masks covering 15--20\% of tokens, with span lengths between 1 and 5. In the Replace setting ($^\text{R}$), the first view is the original sentence and the second replaces 15--20\% of tokens using a pretrained \texttt{bert-base-uncased} proposal model, with replacements constrained to differ from the original tokens. T-JEPA uses random span masking with 1--5 non-overlapping spans per sentence, span lengths between 1 and 5, and the same 15--20\% masking budget.

\paragraph{Optimization.}
All downstream-comparison models are trained for 10 epochs with AdamW, initial learning rate $10^{-4}$, and weight decay $10^{-2}$. We use 100 warmup steps followed by linear learning-rate decay to zero. VICReg uses loss weights $\lambda=25$, $\mu=25$, and $\nu=1$; Barlow Twins uses redundancy-reduction coefficient $\lambda=0.005$; BYOL uses EMA target momentum $\tau_{\mathrm{base}}=0.996$ annealed to $1$ with cosine scheduling and its standard predictor architecture. T-JEPA uses EMA momentum $\tau=0.9996$.

\paragraph{Relation to the diagnostic setup.}
The downstream protocol in Sec.~\ref{sec:downstream} and the diagnostic protocol in Sec.~\ref{sec:arg_setup} serve different purposes and should not be read as competing configurations. The downstream comparison is deliberately \emph{standardized}: T-JEPA and every baseline share the same training budget, warmup length, and optimizer schedule, so no method is given a configuration advantage. The diagnostic setup instead uses a longer, 15-epoch warmup for a different reason: not to induce a failure that would otherwise not occur, but to \emph{slow down its onset} enough that the internal sequence of events -- MI saturation, elevated conditional variance, rank degeneration, and cosine collapse -- can be resolved as distinct, separately measurable stages rather than a single abrupt event. Critically, T-JEPA does not avoid collapse under the shorter, standardized schedule used in Sec.~\ref{sec:downstream}; if anything, it collapses \emph{faster} there, compressing the same five-stage sequence into a window too narrow to dissect on its own (Table~\ref{tab:downstream}). The two setups therefore probe the same underlying failure at two different temporal resolutions: the standardized schedule confirms that the failure occurs under realistic, no-advantage training conditions, while the extended-warmup schedule is a diagnostic instrument used to slow the failure down enough to explain \emph{why} it happens. The longer warmup is a measurement choice, not a cause of the collapse.

\section{Visualization: The Phenomenon of Centroid Degeneration}
\label{app:centroid}
\begin{figure}
    \centering
    \includegraphics[width=\linewidth]{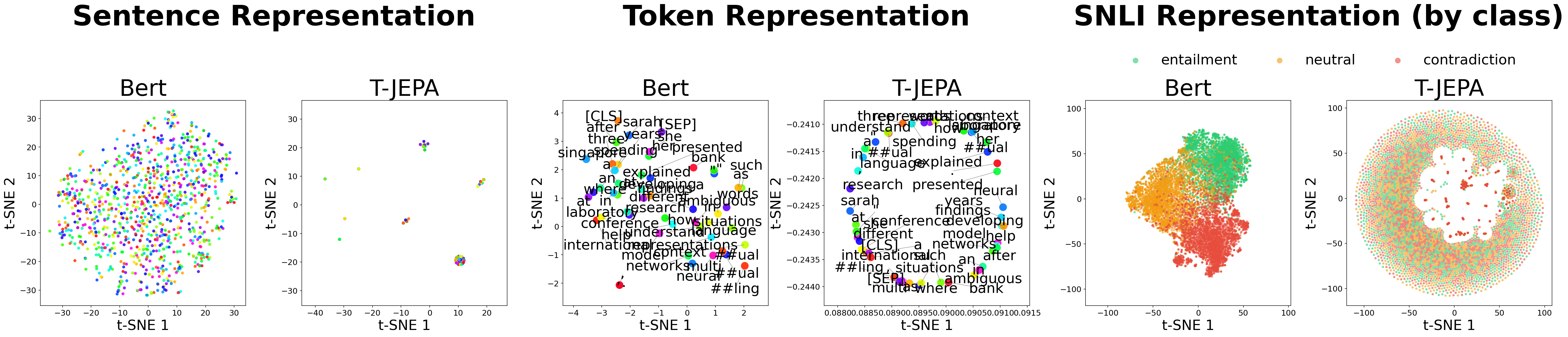}
    \caption{t-SNE visualization of representation geometry for BERT and T-JEPA.
    \textit{Left:} sentence-level embeddings from the raw pretrained last hidden
    state, computed over 10K held-out C4 sentences. \textit{Middle:} token-level
    embeddings for a single representative sentence, with sequential tokens
    connected by edges. \textit{Right:} sentence embeddings after fine-tuning on
    SNLI, colored by label (entailment / neutral / contradiction).}
    \label{fig:presentation}
\end{figure}

Figure~\ref{fig:presentation} gives a qualitative counterpart to the quantitative
collapse diagnostics of Sec.~\ref{sec:arg-I_val}--\ref{sec:arg-III_val}. On raw pretrained C4 sentences (left),
BERT embeddings spread broadly and continuously across the t-SNE plane, whereas
T-JEPA embeddings collapse into a handful of tight, well-isolated clusters --
a low-rank, centroid-like geometry consistent with Prop.~\ref{prop:textcollapse}. The same
degeneration is visible at the token level (middle): BERT places tokens at
distinct, semantically organized positions, while T-JEPA compresses all tokens
of the same sentence into a band spanning a t-SNE range of $\sim$0.003,
indicating near-total loss of token-level distinctions even before any
downstream adaptation. Fine-tuning on SNLI (right) does not undo this geometry:
BERT's embeddings reorganize into class-conditional regions that separate
entailment, neutral, and contradiction examples, whereas T-JEPA's embeddings
remain confined to a narrow, ring-like manifold with all three classes fully
intermixed. This shows that the representational collapse identified during
pretraining is not corrected by fine-tuning, and directly explains T-JEPA's
chance-level SNLI performance in Table~\ref{tab:downstream}.

\section{Visualization: Representation Activation Heatmaps}
\label{app:rep_act}

Figures~\ref{fig:rep_act_3d} and~\ref{fig:rep_act_2d} visualize last hidden-state activations of each context encoder on the same held-out sentence. These plots are intended as qualitative diagnostics of representation geometry, complementing the quantitative rank and cosine-similarity metrics in Sec.~\ref{sec:arg2} and App.~\ref{app:pairwise}.

\begin{figure}[t]
    \centering
    \includegraphics[width=\linewidth]{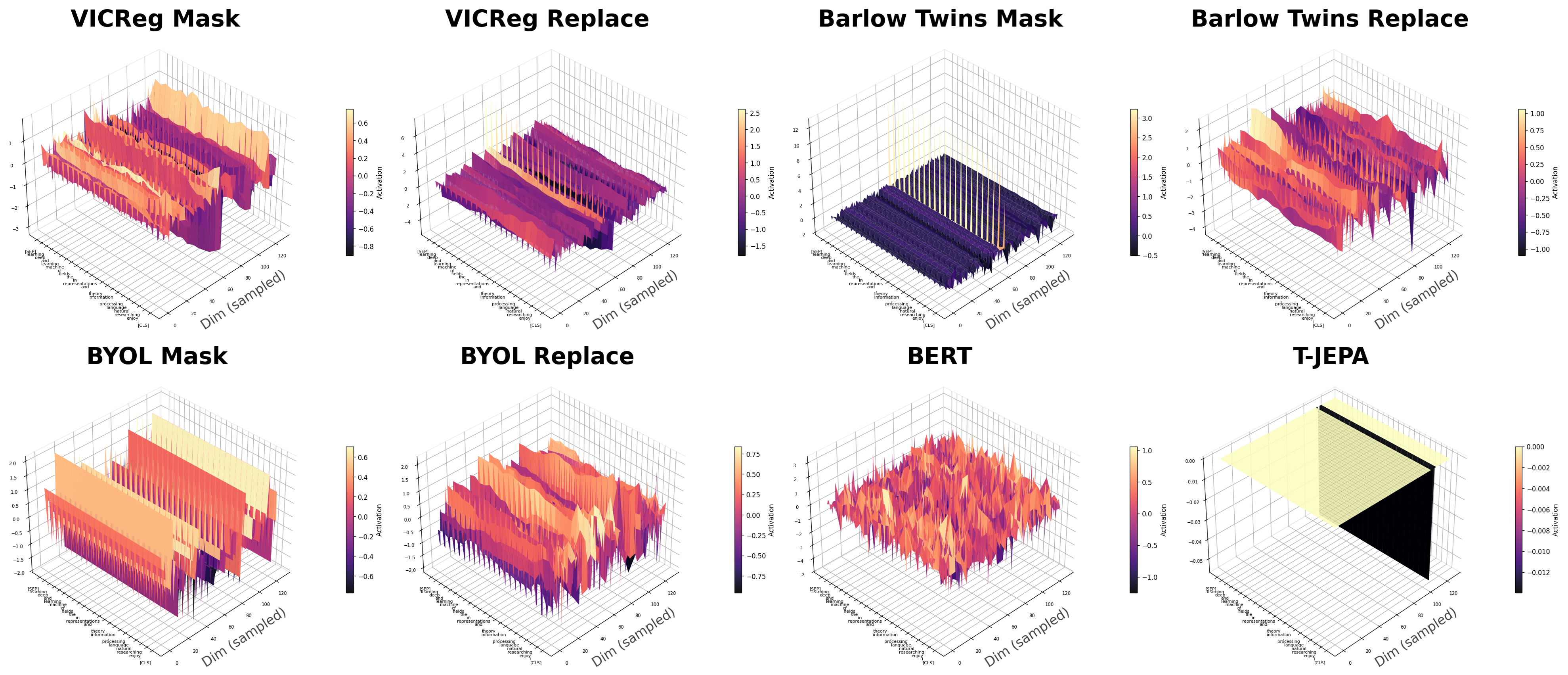}
    \caption{
        \textbf{3D activation surfaces across self-supervised text encoders.}
        Each surface shows last hidden-state activations for a held-out sentence, with token index on one axis and sampled hidden dimensions on the other. BERT exhibits strong token-wise and dimension-wise variation. VICReg, Barlow Twins, and BYOL retain visible activation structure, with variation across both tokens and dimensions. T-JEPA produces a nearly flat, low-amplitude surface, consistent with the collapse diagnostics in App.~\ref{app:collapse}.
    }
    \label{fig:rep_act_3d}
\end{figure}

\begin{figure}[t]
    \centering
    \includegraphics[width=\linewidth]{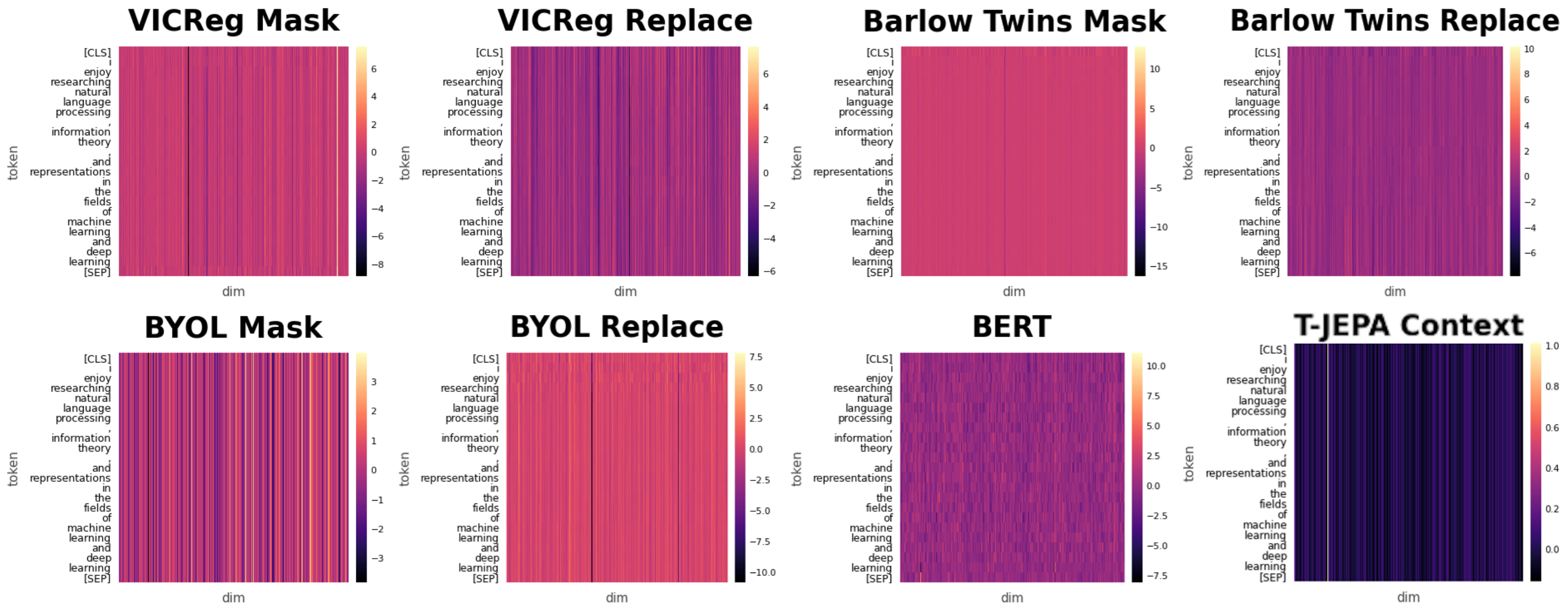}
    \caption{
        \textbf{Activation heatmaps across self-supervised text encoders.}
        Each heatmap shows last hidden-state activations as a token-by-dimension matrix for the same held-out sentence. BERT and the non-JEPA SSL baselines retain structured variation across tokens and dimensions, whereas T-JEPA is nearly uniform and low-amplitude, matching the effective-rank and cosine-collapse results.
    }
    \label{fig:rep_act_2d}
\end{figure}

\paragraph{Qualitative structure.}
A non-degenerate token representation should vary across tokens and use multiple hidden dimensions. BERT shows the strongest activation structure, with clear token-dependent patterns and substantial variation across dimensions. VICReg, Barlow Twins, and BYOL retain moderate structure, suggesting that their objectives preserve some token-level and dimension-level diversity under the same evaluation protocol.

\paragraph{T-JEPA collapse.}
T-JEPA is qualitatively different: its activations are nearly uniform and low-amplitude across both tokens and hidden dimensions. This visual pattern is consistent with the quantitative collapse signatures reported earlier: low effective rank, cosine-similarity saturation, and poor downstream transfer. The heatmaps therefore provide an intuitive view of the same failure mode: T-JEPA does not merely learn weaker features, but compresses the representation geometry itself.

\section{Visualization: T-JEPA Collapse Dynamics}
\label{app:vis_collapse}

Figure~\ref{fig:heatmap_grid} visualizes the last hidden-state activations of the T-JEPA context encoder at several training steps for the same held-out sentence. Rows correspond to tokens and columns to hidden dimensions.

\paragraph{Token-wise and dimensional compression.}
The heatmaps show two qualitative trends. First, token-specific variation is weak throughout training: different tokens share similar activation patterns, with only coarse dimension-wise structure visible. Second, training progressively compresses the activation landscape. The overall activation magnitude shrinks over time, and the remaining variation becomes concentrated in a small subset of hidden dimensions while most dimensions flatten.

This pattern complements the quantitative diagnostics in Sec.~\ref{sec:arg2} and App.~\ref{app:pairwise}. Effective-rank degeneration shows that the covariance spectrum loses dimensions; cosine-similarity saturation shows that different inputs align in representation space; the heatmaps provide a token-level view of the same process. T-JEPA does not merely produce lower-amplitude features. It progressively reduces both token-wise diversity and dimensional usage, consistent with the low-rank collapse predicted by Proposition~\ref{prop:textcollapse}.

\paragraph{Why EMA is insufficient.}
The visualization also clarifies why stop-gradient and EMA do not by themselves prevent this failure mode. EMA stabilizes the target branch by making it track the context encoder slowly, but it does not impose a constraint on the variance or covariance of the representation distribution. If both branches drift toward low-rank activations, the target remains stable while still becoming less informative. Avoiding this failure requires an additional mechanism that preserves variance across examples or dimensions, such as contrastive separation, redundancy reduction, or explicit variance regularization.

\begin{figure}[t]
    \centering
    \includegraphics[width=\linewidth]{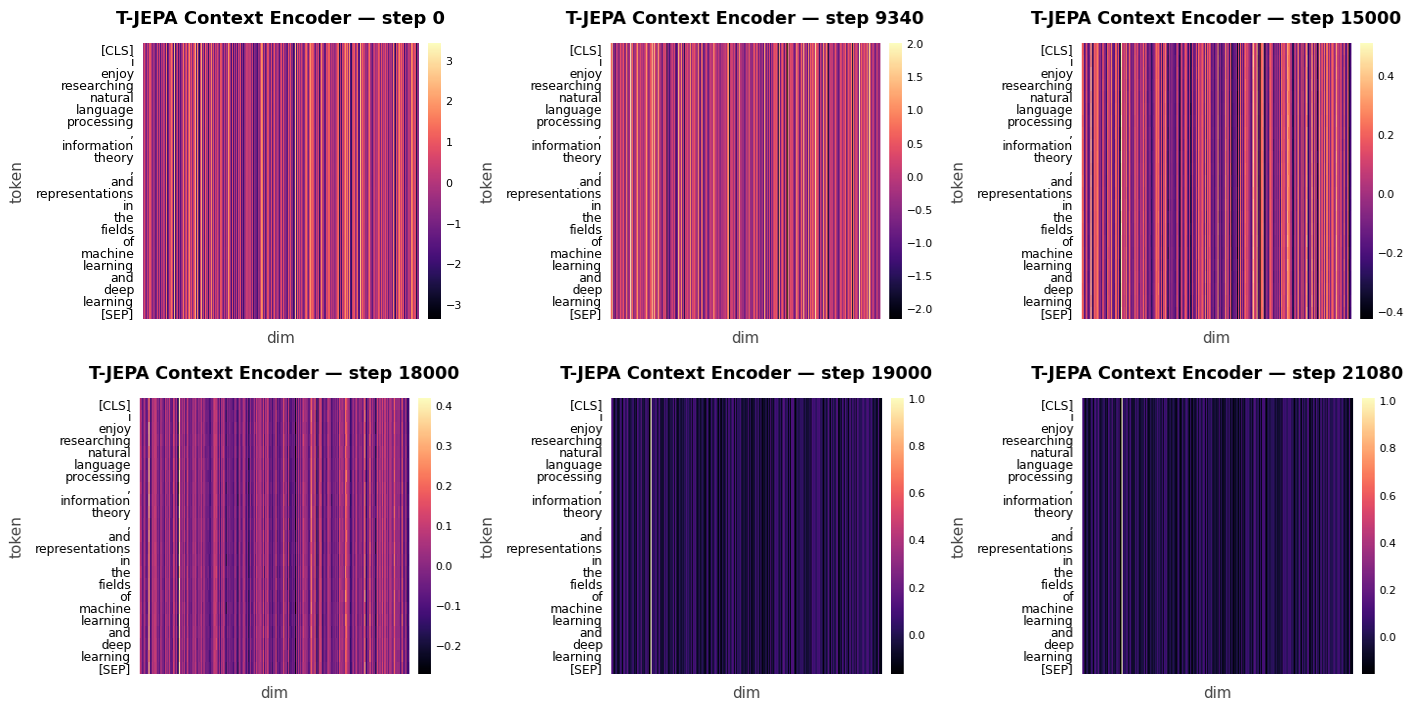}
    \caption{%
    Per-step representation heatmaps of the T-JEPA context encoder's last hidden state for the sentence
    \emph{``I enjoy researching natural language processing, information theory, and representations in the fields of machine learning and deep learning''}
    at steps $0$, $9{,}340$, $15{,}000$, $18{,}000$, $19{,}000$, and $21{,}080$.
    Rows denote tokens and columns denote hidden dimensions.%
    }
    \label{fig:heatmap_grid}
\end{figure}

\section{Comparison with JEPA-Augmented Language Models}
\label{app:discussion-llm-jepa}

Our analysis shows that a deterministic, squared-error latent
prediction objective is mismatched to token-level text: it collapses
ambiguous completions into a centroid corresponding to no real token.
This might seem to conflict with two recent systems, LLM-JEPA and
DLLM-JEPA~\citep{huang2026llmjepa, nam2026dllmjepa}, which report gains
from adding a JEPA-style loss to language model training. The conflict
is only apparent: in both systems the JEPA term never trains alone, and
the reported ablations that remove or weaken the generative objective
are themselves evidence for our thesis.

\begin{table}[t]
\centering
\caption{Comparison of the three text-oriented JEPA studies. T-JEPA is
the direct standalone baseline: JEPA applied to text with the same
backbone and recipe as I-JEPA, with no auxiliary generative loss.}
\setlength{\tabcolsep}{6pt}
\label{tab:jepa-text-comparison}
\setlength{\tabcolsep}{10pt}
\begin{tabular}{lccc}
\toprule
Property & T-JEPA (Our) & LLM-JEPA & DLLM-JEPA \\
\midrule
Encoder (BERT-style) backbone & $\checkmark$ & $\times$ & $\times$ \\
Autoregressive decoder backbone & $\times$ & $\checkmark$ & $\times$ \\
Masked diffusion LM backbone & $\times$ & $\times$ & $\checkmark$ \\
JEPA is sole training objective & $\checkmark$ & $\times$ & $\times$ \\
JEPA is auxiliary to a generative loss & $\times$ & $\checkmark$ & $\checkmark$ \\
Views from masking a single input & $\checkmark$ & $\times$ & $\checkmark$ \\
Requires external paired dataset & $\times$ & $\checkmark$ & $\times$ \\
EMA target encoder + stop-gradient & $\checkmark$ & $\times$ & $\checkmark$ \\
Generative loss resolves token ambiguity & $\times$ & $\checkmark$ & $\checkmark$ \\
Downstream performance improves & $\times$ & $\checkmark$ & $\checkmark$ \\
\bottomrule
\end{tabular}
\end{table}

\paragraph{LLM-JEPA cannot drop the language-modelling loss.}
\citet{huang2026llmjepa} sweep the generative-loss weight $\gamma$ down
to zero on NL-RX-SYNTH~\citep{locascio-etal-2016-neural} while holding
the JEPA weight $\lambda=1$: at $\gamma=0$ (JEPA loss alone) accuracy is
$0.00\%$, rising monotonically with $\gamma$ to $70.42\%$ at $\gamma=1.0$
versus $57.29\%$ for cross-entropy-only. The authors also report that
minimizing $\mathcal{L}_{\mathrm{LLM}}$ alone does not implicitly
minimize the JEPA loss. This matches our theory: the generative loss
keeps the representation informative about token identity, while JEPA
adds structure on top. It also matches what LLM-JEPA's ``views'' are:
\texttt{Text} and \texttt{Code} are fully-specified sequences, not a
masked span with several live completions, so the token-level ambiguity
of Proposition~\ref{prop:text-nonconcentration} is absorbed by
$\mathcal{L}_{\mathrm{LLM}}$ before JEPA ever touches the representation.

\paragraph{DLLM-JEPA cannot drop the diffusion loss or its asymmetric
views.} DLLM-JEPA's generative term $\mathcal{L}_{\mathrm{diff}}$ is
itself a masked, distributional objective, structurally akin to masked
language modelling, and is never removed in any reported configuration;
\citet{nam2026dllmjepa} instead ablate only the JEPA side. On LLaDA-8B
GSM8K~\citep{cobbe2021trainingverifierssolvemath}, the diffusion-only baseline reaches $39.5\%$/$42.6\%$ (0-/4-shot).
Removing the asymmetric JEPA view construction drops this slightly to
$38.9\%$/$42.5\%$; keeping asymmetric views but replacing the predictor
with identity yields only $43.2\%$/$44.4\%$; the full method reaches
$44.9\%$/$61.3\%$. $\mathcal{L}_{\mathrm{diff}}$ is never weakened in
this ablation — it is the JEPA-side components that prove individually
insufficient. Consistently, the paper's own diagnostics show the pooled
JEPA embeddings' effective rank and cosine diversity are essentially
unchanged from the pre-trained base, since $\mathcal{L}_{\mathrm{diff}}$
already resolves ambiguity on its own.

\paragraph{T-JEPA is the missing ``JEPA-only'' cell.} T-JEPA
(Section~\ref{sec:emp_evidence}) is the natural counterfactual: a pure latent-prediction loss
over masked spans with no distributional term, trained under the same
recipe as I-JEPA. It exhibits the full failure chain, falling to
chance-level classification and zero retrieval — effectively the
$\gamma=0$ cell that \citet{huang2026llmjepa} show collapses to empty
output, and the ``no diffusion loss'' cell that \citet{nam2026dllmjepa}
never actually train.

\paragraph{A general pattern, not one specific to JEPA.} An auxiliary
latent-prediction loss helps once a primary objective already supplies
task-relevant structure, but fails, or is never tried, on its own — a
pattern also seen in multi-task learning~\citep{Caruana1997} and in
auxiliary losses added to reinforcement-learning
rewards~\citep{jaderberg2017reinforcement}. LLM-JEPA and DLLM-JEPA fit
the same template: the generative loss resolves ambiguity and keeps
condition~(P) satisfied, and JEPA then adds structure on top of an
already well-formed representation. Table~\ref{tab:jepa-text-comparison}
summarizes this: the two properties that correlate with downstream
success — a distributional generative anchor and a non-standalone JEPA
objective — are exactly what T-JEPA lacks.

\section{Discussion: SIGReg Prevents Collapse But Does Not Resolve
Predictability or Centroid Degeneracy}
\label{app:sigreg}

SIGReg~\citep{Balestriero2025LeJEPA} is a distribution-matching
regularizer that penalizes deviation of the pooled representation
distribution from an isotropic Gaussian, proposed as a principled,
heuristic-free alternative to stop-gradient asymmetry or explicit
negatives for preventing collapse. We show it resolves Argument~II
but not Arguments~I or~III, because it constrains a \emph{marginal}
statistic while both remaining arguments are properties of the
\emph{conditional} target distribution given a single context.

\paragraph{SIGReg resolves Argument II.}
If the SIGReg penalty is minimized, $\Sigma_z\approx\sigma^2 I_d$ for
some $\sigma^2>0$, so $u^\top\Sigma_z u\approx\sigma^2$ for every unit
vector $u$ and Lemma~\ref{lem:noncollapse} gives
$\lambda_{\min}(\Sigma_z)>0$: the low-rank equivalence-class solution
of Prop.~\ref{prop:textcollapse} is ruled out.

\paragraph{A marginal constraint cannot certify a conditional property.}
Let $\mu(x_C)=\mathbb{E}[z_T^{(j)}\mid x_C]$. By the law of total
variance,
\begin{equation}
    \Sigma_z
    =
    \underbrace{\mathbb{E}_{x_C}\!\left[\operatorname{Cov}(z_T^{(j)}\mid x_C)\right]}_{\text{irreducible (Arg.~III)}}
    +
    \underbrace{\operatorname{Cov}_{x_C}(\mu(x_C))}_{\text{context signal (Arg.~I)}}.
    \label{eq:total-var-sigreg}
\end{equation}
SIGReg constrains only the sum. It has no mechanism that separately
bounds either term, so satisfying it is compatible with any split
between them.

\begin{proposition}[Isotropic marginal, vanishing context signal]
\label{prop:sigreg-blind-to-mi}
Let $z_T^{(j)}=z_{\mathrm{sig}}(x_C)+z_{\mathrm{nui}}$ with
$z_{\mathrm{nui}}\sim\mathcal N(0,\sigma^2 I_d)$ independent of $x_C$
and $\|z_{\mathrm{sig}}(x_C)\|_2\le\epsilon$ a.s. Then
$\Sigma_z\to\sigma^2 I_d$ as $\epsilon\to0$ (SIGReg penalty vanishes),
while $I(z_C;z_T^{(j)})\to0$.
\end{proposition}
\begin{proof}
As $\epsilon\to0$, the between-context term in
Eq.~\ref{eq:total-var-sigreg} vanishes while the within-context term
saturates at $\sigma^2 I_d$, so $z_T^{(j)}$ becomes independent of
$x_C$ and $I(x_C;z_T^{(j)})\to0$; Lemma~\ref{lem:dpi} then gives
$I(z_C;z_T^{(j)})\to0$.
\end{proof}
This is exactly the failure the MI proxy of Sec.~\ref{app:mi_proxy} is designed to
catch, and SIGReg supplies no defense against it: Argument~I remains
open.

\paragraph{Argument III is also untouched.}
Even granting nonzero context signal, the within-context term of
Eq.~\ref{eq:total-var-sigreg} is exactly the quantity
Prop.~\ref{prop:textvar-lower} bounds below under
Assumption~\ref{assump:separated-continuations}, and SIGReg places no
constraint on it. The optimal predictor is still the centroid of
Eq.~\ref{eq:centroid}: SIGReg forbids resolving this term by merging
continuations together, but supplies no alternative mechanism for
representing them as anything other than an averaged point.

\begin{mainresult}[title=Consequence IV: Non-Collapse Is Necessary; Not Sufficient]
SIGReg guarantees $\lambda_{\min}(\Sigma_z)>0$, closing the
collapse route of Prop.~\ref{prop:textcollapse}. It gives no
corresponding guarantee for predictability or centroid degeneracy: an
encoder can satisfy SIGReg exactly while remaining unpredictable from
context (Prop.~\ref{prop:sigreg-blind-to-mi}) and while its
Bayes-optimal predictor still collapses distinct continuations onto a
single non-existent centroid (Prop.~\ref{prop:textvar-lower}).
Non-collapse is necessary but, on text, not sufficient.
\end{mainresult}

\paragraph{Future work.} We predict that T-JEPA trained with an added
SIGReg penalty would keep effective rank and cosine similarity healthy
throughout training, yet still show early MI-proxy saturation,
elevated irreducible variance, and downstream transfer well below
BERT, i.e.\ a representation that passes geometric collapse tests
without becoming predictable or semantically faithful. Testing this
with the full five-metric protocol of Sec.~\ref{sec:emp_evidence} is a natural next step.

\clearpage
\section{Mathematical Notation}
\label{app:notation}

\begin{table}[H]
\centering
\small
\begin{tabular}{c|p{4.4cm}|c|p{4.4cm}}
\toprule
\textbf{Symbol} & \textbf{Description} & \textbf{Symbol} & \textbf{Description} \\
\midrule
\multicolumn{4}{l}{\textit{Architecture and encoders}} \\
$f_\theta$ & Context encoder; $z_C=f_\theta(x_C)$ &
$f_{\bar\theta}$ & EMA target encoder \\
$g_\phi$ & Predictor; $(z_C,p_j)\mapsto \hat z_T^{(j)}$ &
$p_j$ & Positional query for target position $j$ \\
$\tau$ & EMA momentum; $\bar\theta \leftarrow \tau\bar\theta+(1-\tau)\theta$ &
$d$ & Latent representation dimension \\
$\operatorname{sg}(\cdot)$ & Stop-gradient operator &
$\theta,\bar\theta,\phi$ & Context, target, and predictor parameters \\
\midrule
\multicolumn{4}{l}{\textit{Data and masking}} \\
$x$ & Full unmasked input &
$x_C$ & Masked context view \\
$x^{(j)}$ & Unit at position $j$; patch or token &
$C,T$ & Context and target index sets \\
$N$ & Number of input units or contrastive candidates &
$\mathcal{V}$ & Token vocabulary \\
$\mathcal{X}$ & Input space &
$d(i,j)$ & Distance between image patch centers \\
$S$ & Masked text span &
$x_{j\leftarrow v}$ & Sequence with token $v$ placed at position $j$ \\
$x_{S\leftarrow s}$ & Sequence with span completion $s$ placed in $S$ &
$\mathcal{V}_\epsilon(x_C)$ & Plausible token set given context $x_C$ \\
\midrule
\multicolumn{4}{l}{\textit{Representations and objectives}} \\
$z_C$ & Context representation &
$z_T^{(j)}$ & Target representation at position $j$ \\
$z^*$ & Target random variable &
$\bar z^{(j)}(z_C,p_j)$ & Conditional mean target \\
$\mathcal{L}_{\mathrm{JEPA}}$ & JEPA squared-error objective &
$\mathcal{L}_{\mathrm{InfoNCE}}$ & InfoNCE contrastive loss \\
$h_j(v;x_C)$ & Target representation induced by token $v$ &
$h_S(s;x_C)$ & Target representation induced by span completion $s$ \\
\midrule
\multicolumn{4}{l}{\textit{Information-theoretic quantities}} \\
$H(X)$ & Shannon entropy &
$H(X\mid Y)$ & Conditional entropy \\
$I(X;Y)$ & Mutual information &
$\widehat{I}_{\mathrm{NCE}}(z_C;z_T)$ & InfoNCE predictability proxy \\
\midrule
\multicolumn{4}{l}{\textit{Collapse and spectral quantities}} \\
$\Sigma_z$ & Representation covariance matrix &
$\lambda_i$ & Eigenvalue of $\Sigma_z$ \\
$\mathrm{erank}(\Sigma_z)$ & Effective rank of $\Sigma_z$ &
$\Delta$ & Minimum separation between plausible targets \\
$\bar{s}$ & Mean pairwise cosine similarity &
$s_{95},\sigma_s$ & 95th percentile and standard deviation of cosine similarity \\
\midrule
\multicolumn{4}{l}{\textit{Variance quantities}} \\
$\operatorname{Var}(z^*\mid z_C,p_j)$ & Irreducible conditional target variance &
$\sigma_{\mathrm{img}}^2$ & Bound on image conditional target variance \\
$\sigma_{\mathrm{vis}}^2$ & Nonzero visual target variation beyond position &
$\widehat{\operatorname{Var}}(z^*\mid z_C,p_j)$ & Empirical target-variance estimate \\
$K$ & Number of sampled variants or completions &
$w_k$ & Importance weight for sampled completion $k$ \\
\midrule
\multicolumn{4}{l}{\textit{General notation}} \\
$\mathbb{R}^d$ & $d$-dimensional Euclidean space &
$\mathbb{E}[\cdot]$ & Expectation \\
$\mathcal{N}(0,\sigma^2 I)$ & Isotropic Gaussian distribution &
$[N]$ & Index set $\{1,\ldots,N\}$ \\
\bottomrule
\end{tabular}
\caption{Notation used throughout the paper and appendices.}
\label{tab:notation}
\end{table}

\end{document}